\pdfoutput=1
\documentclass{article}

\usepackage[preprint]{neurips_2026}

\usepackage[inline]{enumitem}

\usepackage{tikz}
\usetikzlibrary{arrows.meta, positioning, calc}

\usepackage{amsthm}
\usepackage{thmtools}
\newtheorem{proposition}{Proposition}
\newtheorem{lemma}{Lemma}
\theoremstyle{definition}
\newtheorem{definition}{Definition}

\usepackage{setspace} 
\usepackage{graphicx}
\usepackage{subcaption}
\usepackage{booktabs}
\usepackage{multirow}

\usepackage[utf8]{inputenc} 
\usepackage[T1]{fontenc}    
\usepackage{hyperref}       
\hypersetup{colorlinks=true, linkcolor=blue, citecolor=blue, urlcolor=blue}
\usepackage{url}            
\usepackage{amsfonts}       
\usepackage{nicefrac}       
\usepackage{microtype}      
\usepackage{xcolor}         

\usepackage{amsmath,amsfonts,bm}

\newcommand{\Xy}{X_Y}

\newcommand{\xy}{x_Y}

\newcommand{\ey}{e_Y}
\newcommand{\es}{e_S}

\newcommand{\fw}{f_{\mathbf{w}}}

\newcommand{\Ptr}{P_\text{train}}
\newcommand{\Pte}{P_\text{test}}

\newcommand{\Xcal}{\mathcal{X}}
\newcommand{\Scal}{\mathcal{S}}

\newcommand{\Xcaly}{\mathcal{X_Y}}
\newcommand{\Pc}{\mathcal{P}}

\newcommand{\Fcinv}{\mathcal{F}_{\text{inv}}}
\newcommand{\weightset}{\{w_s\}_{s \in \Scal}}
\newcommand{\wstar}{\mathbf{w}^*}
\newcommand{\Btst}{B_{test}}

\newcommand{\classifier}{\mathcal{M}^c}
\newcommand{\lmodel}{\mathcal{M}^l}

\newcommand{\logit}{\operatorname{logit}}

\def\eqref#1{equation~\ref{#1}}

\def\1{\bm{1}}

\DeclareMathAlphabet{\mathsfit}{\encodingdefault}{\sfdefault}{m}{sl}
\SetMathAlphabet{\mathsfit}{bold}{\encodingdefault}{\sfdefault}{bx}{n}

\DeclareMathOperator*{\argmin}{arg\,min}

\title{Counterfactual Test-Time Augmentation for Shortcut Removal in Language Models}

\author{%
  \textbf{Katie Matton} \\
  MIT \\
  \texttt{kmatton@mit.edu} \\
  \and
  \textbf{Rosalind W. Picard} \\
  MIT \\
  \and
  \textbf{John Guttag} \\
  MIT \\
  \and
  \textbf{Maggie Makar} \\
  University of Michigan \\
}

\begin{document}

\maketitle

\begin{abstract}
Large language models (LLMs) can exploit shortcuts, i.e., correlations present in training data that may not hold at test time. This can lead to poor predictions when correlations shift. Most existing solutions for removing shortcuts operate at training time, requiring labeled training data, model weight access, and significant computational resources. Here, we study the problem of \textit{test-time} shortcut removal. We propose a family of counterfactual test-time augmentation (CTTA) methods: given an example, each method uses an auxiliary model to generate counterfactual edits in which the shortcut value is changed, and then takes a weighted average of the LLM's predictions across these edits. We prove that the optimal shortcut-invariant predictor is a member of this CTTA family, and we characterize its weights. The weights depend on the training distribution, and therefore are, in general, not identifiable from the trained model and unlabeled test data. However, we present assumptions under which the weights can be obtained using the token probabilities of the trained LLM, and propose tests to detect violations of these assumptions. We use these tests to develop Select-CTTA, an aggregate method that selects among CTTA methods. Empirically, we find that Select-CTTA outperforms test-time baselines and approaches the performance of models fine-tuned on shortcut-balanced data.
\end{abstract}

\section{Introduction}

Despite impressive performance on a range of tasks, large language models (LLMs) remain brittle. They have been shown to abandon correct answers when users push back \citep{sharma2024towards}, change answers when multiple-choice options are reordered \citep{pezeshkpour2024large}, and rely on demographic stereotypes when making decisions about people \citep{tamkin2023evaluating}. Each of these failures is related to a common issue: LLMs are prone to relying on \textit{shortcuts}, i.e., correlations in training data that may not hold at inference time. As a result, LLMs use heuristics, such as \textit{prefer option A} and \textit{agree with the user}, that reduce training loss but lead to errors during deployment \citep{zheng2024large, sharma2024towards}. Reliance on shortcuts is a major impediment to LLM safety and reliability.

Most existing methods for removing shortcuts in language models are designed to be applied at training time. However, training an LLM requires many resources, including labeled training data, access to model weights, and substantial compute. Users often lack these resources. Hence, here we study a different approach: \textit{removing shortcuts from LLMs at test-time without access to model weights or labeled data.}

We propose a family of test-time shortcut removal methods that we call counterfactual test-time augmentation (CTTA).  Given a test example, each CTTA method uses an auxiliary model to create counterfactual edits in which the shortcut value is changed. We then take a weighted average of the LLM’s predictions across the edits. As we show in this work, if the chosen weights do not depend on the shortcut value observed in the example, the resulting predictor is shortcut-invariant.

Among the set of predictors that are shortcut invariant, we want the one with the lowest error, i.e., the \textit{optimal shortcut-invariant predictor}. We prove that under a common model of shortcut learning, the optimal shortcut-invariant predictor is a member of the CTTA family. The optimal weights depend on the training distribution and therefore, in general, cannot be identified without access to the training data. However, we present a set of assumptions under which the weights can be obtained using the observable token probabilities of the trained LLM. We call the resulting method Model-Weighted CTTA. 

Since the assumptions needed to recover the optimal weights are not always met in practice, we also study simple fixed weighting schemes, such as uniform weighting, that do not require a weight estimation step. We prove fixed weighting schemes have a desirable property: they achieve the same AUROC as the optimal shortcut-invariant predictor, although they are miscalibrated. To select among CTTA methods, we propose Select-CTTA, a procedure that performs diagnostic tests to detect violations of the Model-Weighted CTTA assumptions. When the checks pass, Select-CTTA selects Model-Weighted CTTA; otherwise, it selects uniform weights. 

We evaluate the CTTA methods on two text classification benchmarks and six LLMs, spanning multiple model families and sizes. Select-CTTA consistently achieves the best or near-best performance among CTTA methods and outperforms test-time calibration baselines. In many cases, it approaches the performance of models trained on gold-standard balanced data, while requiring substantially fewer resources. 

Our contributions are:
\begin{enumerate*}[label=\textbf{(\arabic*)}]
\item We propose a family of Counterfactual Test Time Augmentation (CTTA) methods as a solution for test-time shortcut removal in LLMs. 
\item We theoretically analyze the optimality of our proposed methods, providing an understanding of when they match (and fail to match) the performance of the optimal shortcut-invariant predictor.
\item We introduce Select-CTTA, a procedure that adaptively selects between model-based and uniform CTTA weights using label-free diagnostic tests.
\item  We demonstrate that Select-CTTA outperforms test-time baselines and often approaches the performance of fine-tuning on shortcut-balanced data at substantially lower cost. 
\end{enumerate*}

\section{Problem Setting}\label{sec:problem-setup}
We now formalize the problem of test-time shortcut removal in LLMs. We describe our assumed data generating process, define our objective, and describe the resources we assume are accessible at test-time.

\textbf{Data Generating Process.} Following prior work \citep{makar2022causally, feder2023causal}, we assume that the data are generated according to the causal DAG in Figure~\ref{fig:causal-dag}. At test time, we observe texts $X \in \Xcal$, where $\Xcal$ is the space of finite text sequences. In this anti--causal setting, each text is generated from a binary label $Y \in \{0,1\}$ and a categorical shortcut feature $S \in \Scal$. We write the label-relevant content of $X$ as $\Xy \in \Xcaly$ and assume that $Y$ affects $X$ only through $\Xy$. For example, consider biographical texts $X$ with occupation for labels $Y$ and gender as the shortcut $S$. The label $Y$ determines the occupation-related content $\Xy$ (e.g., job responsibilities), while the shortcut is reflected in the names and pronouns used. We assume that the shortcut feature and its possible values $\Scal$ are known. We also assume that $\Xy$ and $S$ can, in theory, be recovered from $X$ through functions $\Xy := \ey(X)$ and $S := \es(X)$. However, we do not assume that these functions are known.

We denote the distribution of data that the LLM was trained on as $\Ptr$ and the distribution of test data as $\Pte$. We assume that $\Pte$ differs from $\Ptr$ only in terms of the association between the shortcut $S$ and label $Y$, with everything else fixed. Formally, we assume that $\Pte \in \Pc$ for $\Pc$ defined as:
\[
\begin{aligned}
\Pc := \bigl\{\,
  &\Ptr(Y)\, P(S \mid Y) \\
  &\Ptr(\Xy \mid Y)\,
  \Ptr(X \mid \Xy, S)
\,\bigr\},
\end{aligned}
\]
where $P(S \mid Y)$ is an arbitrary conditional distribution. While the test conditional is unrestricted, we assume the training distribution satisfies the following positivity conditions: $\Ptr(Y=y)>0$, $\Ptr(S=s\mid Y=y)>0$ and $\Ptr(\Xy=\xy\mid Y=y)>0$ for all $y\in\{0,1\}$, $s\in\Scal$, and $\xy$ in the support of $\Xy$.

\begin{figure}[t]
  \centering
  \begin{tikzpicture}[
      >={Stealth[length=2.0mm]},
      node distance=6mm and 10mm
    ]
    \tikzset{rv/.style={circle, draw, minimum size=9.5mm, inner sep=1pt}}
    \node[rv] (S)  {$S$};
    \node[rv] (Y)  [right=of S] {$Y$};
    \node[rv] (Xy) [below=of Y] {$X_y$};
    \node[rv, fill=gray!30] (X)  [below=of S] {$X$};
    \draw[<->, dashed] (S) -- (Y);
    \draw[->] (S) -- (X);
    \draw[->] (Y) -- (Xy);
    \draw[->] (Xy) -- (X);
  \end{tikzpicture}
  \caption{Causal graph of the data-generating process. The label $Y$ and shortcut $S$ may be correlated (dashed bidirected edge). $Y$ affects the text $X$ only through $\Xy$. Shading indicates observed at test-time.}
  \label{fig:causal-dag}
\end{figure}
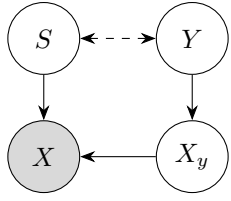

\textbf{Objective.} We consider predictors $f:\Xcal\to[0,1]$ that map a text $x$ to the predicted probability that $Y=1$. We define the risk of $f$ under a distribution $P$ as
\[
R_P(f) := \mathbb{E}_{(X,Y)\sim P}\big[\ell(f(X),Y)\big],
\]
where $\ell$ is the logistic loss. We want a predictor whose risk does not depend on the relationship between the shortcut and the label, since this may change arbitrarily at test time. We denote the set of \textit{shortcut-invariant} predictors as: \[
    \Fcinv := \bigl\{
        f : R_P(f) = R_{P'}(f)
        \text{ for all } P, P' \in \Pc
    \bigr\}.
\] Among the predictors in $\Fcinv$, we seek the \textit{optimal shortcut-invariant predictor}, defined as:
\[
    f^* := \argmin_{f \in \Fcinv} R_P(f),
\]
where \(P\) is any distribution in \(\Pc\). Under the assumed data generating process, $f^*$ is equal to the conditional distribution of the label given the label-relevant content of the text \citep{makar2022causally}:  \[f^*(x) = P\big(Y = 1 \mid \Xy = \ey(x)\big),\] where $P$ is any distribution in \(\Pc\).

\textbf{Resources.} We assume access to three resources at test time.

\textit{(1) Trained LLM:} an LLM $\mathcal{M}$ that we can query in two modes. In \textit{classification} mode, $\mathcal{M}^c(x)\in[0,1]$ denotes the LLM's predicted probability that $Y=1$ given $x$. In \textit{likelihood} mode, $\mathcal{M}^l(x)$ denotes the probability assigned by the LLM to the input text sequence. For a text $x=(x_1,\ldots,x_T)$, this is computed autoregressively from the LLM's token probabilities as $\mathcal{M}^l(x) := \prod_{t=1}^{T} \mathcal{M}(x_t \mid x_{<t})$. Throughout our theoretical analysis, we make the assumption that the classification mode recovers the true training conditional, i.e., $\mathcal{M}^c(x) = \Ptr(Y=1\mid X=x)$. When analyzing Model-Weighted CTTA (Section~\ref{subsec:mw-ctta-def}), we additionally assume that the likelihood model recovers the training marginal of $X$, i.e., $\mathcal{M}^l(x) = \Ptr(X=x)$.

\textit{(2) Counterfactual shortcut editor:} a model $\tau: \Xcal \times \Scal \rightarrow \Xcal$ that, given a text $x$ and target shortcut value $s$, generates a counterfactual version of $x$ in which the shortcut is set to $s$ and the remaining content is unchanged. Formally, for $x' = \tau(x, s)$, it holds that $\es(x') = s$ and $\ey(x) = \ey(x')$. We additionally assume that all counterfactual edits lie in the support of the
training distribution. In practice, we implement $\tau$ with a language model prompted to change only the parts of the text that express the shortcut.

\textit{(3) Unlabeled test data:} a set of $N$ unlabeled texts  $\Btst := \{x_i\}_{i=1}^N$ where each $x \sim \Pte(X)$.

\section{Methods}\label{sec:methods}
To address the problem of test-time shortcut removal, we propose a family of methods that we call Counterfactual Test-Time Augmentation (CTTA). We first define this family and show that the optimal shortcut-invariant predictor belongs to it (Section~\ref{subsec:ctta-fam}). We then introduce Model-Weighted CTTA and establish conditions under which it recovers the optimal predictor (Section~\ref{subsec:mw-ctta-def}). Next, we analyze Fixed-Weight CTTA and show that it achieves the same AUROC as the optimal predictor, despite generally being miscalibrated (Section~\ref{subsec:fw-ctta-def}). Finally, we introduce Select-CTTA, a label-free procedure that adaptively selects between Model-Weighted and Fixed-Weight CTTA (Section~\ref{subsec:select-ctta-def}). 

\subsection{The CTTA Family}\label{subsec:ctta-fam}
Given a test example, each CTTA method uses the editor to produce a counterfactual edit for each shortcut value, queries the trained LLM on each edited example, and returns a weighted average of the resulting predictions. Formally, we define a CTTA predictor as:
\begin{definition}[CTTA predictor]\label{def:tta}
Let $\mathbf{w} := \weightset$ be a set of weight functions $w_s : \Xcal \to [0,1]$ satisfying $\sum_{s \in \Scal} w_s(x) = 1$ for every $x \in \Xcal$. The CTTA predictor with weights $\mathbf{w}$ is
\begin{equation}\label{eq:tta}
    \fw(x) := \sum_{s \in \Scal} w_s(x)\, \classifier\big(\tau(x, s)\big).
\end{equation}
\end{definition}
Next, we characterize the shortcut-invariance and optimality of CTTA predictors. 

\textbf{Shortcut Invariance.} Evaluating the LLM on counterfactual edits removes dependence on the shortcut value originally observed in $x$, making the weight functions the only potential source of shortcut dependence. Therefore, a CTTA predictor is shortcut-invariant if its weight functions depend on the text only through its label-relevant content, i.e., $w_s(x) = w_s(x')$ whenever $\ey(x) = \ey(x')$ (proof in Appendix~\ref{app:ctta-si}). Fixed weights satisfy this condition automatically, while Model-Weighted CTTA satisfies it under the assumptions introduced below.

\textbf{Optimality.} The optimal shortcut-invariant predictor can be recovered through CTTA with weights $w^*_s(x) := \Ptr\big(S = s \mid \Xy = \ey(x)\big)$. Intuitively, the optimal weights marginalize out the shortcut $S$ while conditioning on the label relevant content $\Xy$. We state this formally in Proposition~\ref{prop:opt-inv-pred} and provide a proof in Appendix~\ref{app:ctta-opt}.

\begin{restatable}{proposition}{OptCTTAPredictor}\label{prop:opt-inv-pred}
For each $s \in \Scal$, let $w^*_s(x) := \Ptr\big(S = s \mid \Xy = \ey(x)\big)$, and write $\wstar = \{w^*_s\}_{s \in \Scal}$. Then for all $x \in \Xcal$, $f_{\wstar}(x) = f^*(x)$.
\end{restatable}

The optimal weights depend on the training distribution and are generally not identifiable from the trained LLM and unlabeled test data alone. We next present assumptions under which they can be recovered using the LLM's token probabilities.

\subsection{Model-Weighted CTTA}\label{subsec:mw-ctta-def}
We consider a CTTA predictor that weights each counterfactual edit by its probability under the trained LLM. We define \textit{Model-Weighted CTTA} weights as
\[
w_s^{\mathrm{MW}}(x)
:=
\frac{\mathcal{M}^l(\tau(x,s))}
{\sum_{s' \in \Scal}\mathcal{M}^l(\tau(x,s'))}.
\]
We next give two assumptions under which these weights recover the optimal CTTA weights.

\textit{(A1) Idealized Likelihood Model.}
The likelihood mode of the trained LLM recovers the
training marginal of $X$, i.e.,
$\mathcal{M}^l(x) = \Ptr(X=x)$ for all $x \in \Xcal$. 

Together with our standing assumption that the classification mode
satisfies $\mathcal{M}^c(x) = \Ptr(Y=1 \mid X=x)$, (A1) implies that the likelihood and classification modes both reflect the same underlying training distribution $\Ptr$. We refer to this property as \textit{likelihood-classifier alignment}.

\textit{(A2) Idealized Editor.} For each example $x$, the editor $\tau$ creates edits whose probabilities under the training distribution differ only because of the shortcut value. That is, any residual variation in the text (e.g., length, fluency, or phrasing) contributes equally to the probability of each edit. Formally, for all $s,s'\in\Scal$,
\[
\begin{aligned}
&\Ptr\big(X=\tau(x,s)\mid \Xy=\ey(x),S=s\big)\\
&\qquad=\Ptr\big(X=\tau(x,s')\mid \Xy=\ey(x),S=s'\big).
\end{aligned}
\]

Under (A1) and (A2), the likelihood of each counterfactual edit is
proportional to the joint training probability of its label-relevant
content and shortcut value:
\[
\lmodel\big(\tau(x,s)\big)
\propto \Ptr\big(\Xy=\ey(x), S=s\big),
\]
where the proportionality constant does not depend on $s$.
Normalizing therefore gives:
\[
w_s^{\mathrm{MW}}(x)
=
\Ptr\big(S=s\mid\Xy=\ey(x)\big)
=
w_s^*(x).
\]

Thus, Model-Weighted CTTA uses the optimal CTTA weights and recovers the
optimal shortcut-invariant predictor. We state this formally in
Proposition~\ref{prop:mw-ctta-opt} and provide a proof in
Appendix~\ref{app:mw-ctta-opt}.

\begin{restatable}{proposition}{ModelWeightedOptimality}
\label{prop:mw-ctta-opt}
Suppose (A1) and (A2) hold. For each $s \in \Scal$, define
\[
w_s^{\mathrm{MW}}(x)
:=
\frac{\mathcal{M}^l(\tau(x,s))}
{\sum_{s' \in \Scal}\mathcal{M}^l(\tau(x,s'))},
\]
and write
$\mathbf{w}^{\mathrm{MW}}
:= \{w_s^{\mathrm{MW}}\}_{s \in \Scal}$.
Then, for all $x \in \Xcal$, $f_{\mathbf{w}^{\mathrm{MW}}}(x)=f^*(x)$.
\end{restatable}

\subsection{Fixed-Weight CTTA}\label{subsec:fw-ctta-def}  Since the assumptions of Model-Weighted CTTA might not hold in practice, we also consider Fixed-Weight CTTA methods, which do not need a weight estimation step. We define \textit{fixed weights} as those that do not depend on the input example $x$, i.e., $w_s(x) = w_s$ for all $x \in \Xcal$. Two natural choices are \textit{Uniform CTTA}, which assigns equal weight to each shortcut value, $w_s=1/|\Scal|$, and \textit{Neutral CTTA}, which places all weight on a designated neutral value $s_0$, i.e., $w_{s_0}=1$ and $w_s=0$ for $s\neq s_0$. For example, when gender is the shortcut, the neutral edit may remove all references to gender from the text.

To understand the relationship between fixed weight predictors and $f^*$, we first consider the prediction that the trained LLM $\classifier$ makes on a single counterfactual edit. In our assumed DAG (Figure~\ref{fig:causal-dag}), the label-relevant content $\Xy$ and the shortcut $S$ are conditionally independent given the label $Y$. As a result, the model's log-odds decompose into a term determined by the label-relevant content and a term determined only by the shortcut value (proof in Appendix~\ref{app:s-edit-lo}):
\begin{equation}\label{eq:edit-logit-decomp}
\begin{aligned}
\operatorname{logit}\classifier(\tau(x,s))
&=
\operatorname{logit}f^*(x)
+
\lambda(s),\\
\lambda(s)
&:=
\log
\frac{\Ptr(S=s\mid Y=1)}
     {\Ptr(S=s\mid Y=0)}.
\end{aligned}
\end{equation}
This shows that $\classifier$ shifts the optimal log-odds by a constant that depends only on the shortcut value: $\lambda(s) > 0$ if $s$ is more common among training examples with label $1$ than among those with label $0$, while $\lambda(s)=0$ if $s$ is equally prevalent in the two classes. Substituting the decomposition into Equation~\ref{eq:tta}, we obtain
\begin{equation}\label{eq:fixed-weight-function}
\fw(x)
=
\sum_{s\in\Scal}
w_s\,
\sigma\!\left(
\operatorname{logit}f^*(x)+\lambda(s)
\right),
\end{equation}
where $\sigma(z)=(1+e^{-z})^{-1}$. For any fixed weights $\mathbf{w}$, the right-hand side is strictly increasing in $\operatorname{logit}f^*(x)$. Thus, fixed-weight CTTA preserves the ranking induced by $f^*$ and obtains optimal AUROC. We state this formally in Proposition~\ref{prop:fixed-weight-auroc} and provide a proof in Appendix~\ref{app:fw-ctta-rank-opt}.
\begin{restatable}{proposition}{FixedWeightAUROC}
\label{prop:fixed-weight-auroc}
For any fixed weights $\mathbf{w}=\{w_s\}_{s\in\Scal}$,
\[
\fw(x) > \fw(x')
\quad\Longleftrightarrow\quad
f^*(x) > f^*(x')
\]
for all $x,x'\in\Xcal$. Consequently, $\fw$ and $f^*$ achieve the same AUROC under every $P\in\Pc$.
\end{restatable}

\subsection{Select-CTTA}\label{subsec:select-ctta-def}
Model-Weighted CTTA can recover the optimal shortcut-invariant predictor, but only under assumptions that may fail in practice. We therefore introduce Select-CTTA, a procedure for choosing between Model-Weighted and fixed-weight CTTA at test time. It uses a batch of unlabeled test data to perform three diagnostic tests. The first examines the strength of the shortcut, while the second two provide
complementary diagnostics for violations of the Model-Weighted CTTA
assumptions. If all tests pass, it uses Model-Weighted CTTA. Otherwise, it uses Uniform CTTA. We prefer Uniform to Neutral CTTA because identifying a truly neutral shortcut value may be infeasible in practice. We provide details on how to implement the tests in Appendix~\ref{app:select-ctta}.

\textbf{Test 1: Shortcut Strength.} If the shortcut has a weak association with the label, then $\lambda(s)$ is small, and fixed weight predictors are near-optimal (see Equation~\ref{eq:fixed-weight-function}). In this regime, estimating the optimal weights offers little potential benefit and can introduce additional error, so we choose Uniform CTTA.

Ideally, to assess shortcut strength, we'd examine $\lambda(s)$. However, the values $\lambda(s)$ are not directly estimatable because $f^*$ is unknown. Instead, we estimate shortcut strength using the model's predictions on counterfactual edits. By Equation~\ref{eq:edit-logit-decomp}, for any
$s,s'\in\Scal$,
\[
\operatorname{logit}\classifier(\tau(x,s))
-
\operatorname{logit}\classifier(\tau(x,s'))
=
\lambda(s)-\lambda(s').
\]
Thus, $|\lambda(s)-\lambda(s')|$ measures how much changing the shortcut from $s'$ to $s$ shifts the model's prediction, making it a measure of shortcut strength. For each test example $x_i$, we define the shift between shortcut values $s$ and $s'$ as
\[
d_i(s,s')
:=
\operatorname{logit}\classifier(\tau(x_i,s))
-
\operatorname{logit}\classifier(\tau(x_i,s')).
\]
We compute the mean shift, $\Delta_{s,s'} := \frac{1}{N}\sum_{i=1}^N d_i(s,s')$. We define the estimated shortcut strength as the largest absolute mean shift over shortcut value pairs, $\Delta := \max_{s,s'\in\Scal} \left|\Delta_{s,s'}\right|$. We use Uniform CTTA when $\Delta < \delta_{\mathrm{str}}$, where $\delta_{\mathrm{str}}$ is a shortcut-strength threshold. We discuss how to choose $\delta_{\mathrm{str}}$, and the corresponding parameters for the other tests, in Appendix~\ref{app:select-ctta}.

\textbf{Test 2: Weight Collapse.} We check whether the Model-Weighted CTTA weights have collapsed, meaning that they systematically assign most of the weight to the same shortcut value across examples. Though collapse does not necessarily indicate an assumption violation, assumption violations can lead to collapse. A violation of (A1) can cause collapse if the discrepancy between the LLM's likelihood and the true training probability varies systematically with the shortcut value.  A violation of (A2) can cause collapse if the editor systematically produces more probable edits for one shortcut value. For example, edits to one value may involve only a small, natural substitution, while edits to another may require longer or less fluent rewrites.

We measure weight collapse by computing, for each shortcut value, its median weight across the test batch, and taking the largest of these values:
\[
C
:=
\max_{s\in\Scal}
\operatorname{median}_{i=1,\ldots,N}
w_s^{\mathrm{MW}}(x_i).
\]
If $C$ exceeds a threshold $\delta_{\mathrm{col}}$, we treat the weights as collapsed and use Uniform CTTA. We additionally apply an analogous check within the pair of shortcut values with the largest estimated shortcut shift, flagging cases where one member systematically receives much more weight than the other. This can reveal pairwise collapse that is obscured when weight is spread across other shortcut values. 

Although this test catches some cases in which (A1) or (A2) is violated, it is not comprehensive: the weights can be incorrect  without being collapsed. Test 3 provides a complementary check of assumption violations.

\textbf{Test 3: Likelihood-Classifier Alignment.} Under our problem setup, assumption (A1) implies that the likelihood mode $\mathcal{M}^l$ and the classification mode $\mathcal{M}^c$ of the LLM reflect the same underlying training distribution. Thus, the shortcut-label relationship encoded by each should be consistent. For example, if changing the shortcut from $s'$ to $s$ makes $\mathcal{M}^c$ more likely to predict $Y=1$, then the probability that $\mathcal{M}^l$ puts on $s$ relative to $s'$ should be higher for examples whose label-relevant content is more indicative of class $1$.

We focus on the pair of shortcut values $(s,s')$ with the largest estimated shortcut shift, since it provides the strongest signal for testing this alignment. Without loss of generality, we order the pair so that changing the shortcut from $s'$ to $s$ increases the classifier's predicted probability of label 1. We would like to rank examples by how strongly their label-relevant content indicates class $1$, i.e., by $f^*(x)$. By Equation~\ref{eq:edit-logit-decomp}, any single counterfactual edit induces the same ranking as $f^*(x)$ and can be used. In practice, we use the average logit across edits for a more stable ranking. For each example, we measure the relative weight on $s$ as,
\[
r_s(x)
:=
\frac{w_s^{\mathrm{MW}}(x)}
{w_s^{\mathrm{MW}}(x)+w_{s'}^{\mathrm{MW}}(x)}.
\]
Under likelihood-classifier alignment and (A2), $r_s(x)$ should increase with $f^*(x)$. We therefore compute the Spearman rank correlation $\rho$ between $r_s(x)$ and the average classifier logit across edits (which has the same ranking as $f^*$). If $\rho < \delta_{\mathrm{align}}$, where $\delta_{\mathrm{align}}>0$ is an alignment threshold, the test fails and Select-CTTA uses Uniform CTTA. Such misalignment can arise if the likelihood mode encodes a different shortcut-label relationship from the classifier, violating (A1), or if editor errors distort the likelihoods in an example-dependent way, violating (A2).

\section{Experiments}

\subsection{Tasks}  
We examine two different binary classification tasks. Data processing details are in Appendix~\ref{app:data}.

\textbf{Math Grading (Sycophancy).} LLMs have been shown to exhibit sycophantic behavior when providing feedback on math solutions \citep{sharma2024towards}. In this task, the input $X$ is a problem and a proposed solution, and the label $Y \in \{0, 1\}$ indicates whether the solution is correct. We use the PRM800K dataset \citep{lightman2023lets}, which contains problems from MATH \citep{hendrycksmath2021} with model-generated solutions and human-provided correctness labels. As the shortcut $S$, we insert one of three opinion statements at the end of the input: \textit{I think this solution is correct}, \textit{I think this solution is incorrect}, and \textit{I'm unsure whether this solution is correct}. 

\textbf{Occupation Classification (Gender).} We use the Bias in Bios dataset \citep{de2019bias}, a text classification benchmark introduced to study gender bias in occupation prediction. We use biographies labeled as attorney or paralegal. Each input $X$ is a biography, the label $Y \in \{0, 1\}$ is occupation, and the shortcut $S \in \{\text{man}, \text{woman}\}$ is the gender of the biography subject.

\subsection{Experimental Settings}\label{subsec:exp-settings}

\textbf{Training Distributions.} In our first experiments, we fine-tune models on data with a controlled shortcut strength given by  $p \in [0.5, 1]$. For a designated shortcut value $s$, we set $\Ptr(S = s \mid Y = 0) = p$ and  $\Ptr(S = s \mid Y = 1) = 1 - p$. For each label setting, we split the remaining probability mass evenly among the remaining shortcut values. Our main experiments use $p=0.95$ and $\Ptr(Y=1)=0.5$. Appendix~\ref{app:shortcut-strength} reports results across a range of shortcut strengths.

\textbf{Test Distributions.} For each training distribution, we reverse the shortcut-label association at test time:
$\Pte(S=s\mid Y=1)=\Ptr(S=s\mid Y=0)$ and
$\Pte(S=s\mid Y=0)=\Ptr(S=s\mid Y=1)$. We use $\Pte(Y=1)=0.5$. For the math task, we use 2,600 examples, and for the occupation task, we use 582 examples. Appendix~\ref{app:test-distr} reports results under milder distribution shifts.

\textbf{Models.} We use six instruction-tuned LLMs from three open-weight model families: Gemma 3 (4B and 12B; \citealp{team2025gemma}), Mistral (7B and Nemo 12B; \citealp{jiang2023mistral7b}), and Qwen 2.5 (1.5B and 7B; \citealp{yang2024qwen2}). 

\textbf{Model Settings.} We evaluate each LLM under three training conditions: \textit{all-tokens finetuning}, where we fine-tune the model to predict both the input tokens and classification label; \textit{completion-only finetuning}, where we fine-tune only on the classification label; and \textit{zero-shot}, where we do not fine-tune at all.

\textbf{Baselines.} We compare the CTTA methods to each other and to existing test-time methods. For the math task, we consider two Neutral CTTA variants: \textit{unsure}, which uses the unsure opinion statement, and \textit{removal}, which omits an opinion statement. For occupation, Neutral CTTA uses a language model to remove gender references. In our zero-shot experiments, we compare to Out-of-Context (OOC) prompting \citep{cotta2024test}, a shortcut-robustness method designed for the zero-shot setting. OOC is similar to Uniform CTTA, but it generates edits with the predictor LLM itself, and its editing prompt instructs the model to preserve a stratification variable, which in the anti-causal setting of Figure~\ref{fig:causal-dag} is the LLM's predicted label. We also compare to Batch Calibration (BC) \citep{zhou2024batch} and Contextual Calibration (CC) \citep{zhao2021calibrate}, two test-time calibration methods. They are not designed specifically for shortcut removal, so we additionally introduce shortcut-conditioned variants of each, which perform calibration separately for each shortcut value. In our finetuning experiments, we also report the performance of models fine-tuned on balanced data. We consider this a reference rather than a baseline, since it assumes access to labeled training data and model weights.

\textbf{Counterfactual Editor.} For the math task, the shortcut edits are straightforward and can be applied in a rule-based manner. For the occupation task, we perform gender edits using Claude Opus (details in Appendix~\ref{app:exps-editor}).

\textbf{Hyperparameters.} In Select-CTTA, we set $\delta_{\mathrm{str}} = 0.5$, $\delta_{\mathrm{col}} = 0.8$, and $\delta_{\mathrm{align}} = 0.5$. During model inference, we use a temperature of $1$ in all experiments.

\textbf{Metrics.} We evaluate how well predictions rank examples using AUROC. We measure calibration using Expected Calibration Error (ECE), with results for additional metrics in Appendix~\ref{app:metrics}.

\subsection{Results}

\begin{table}[t]
\centering
\footnotesize
\caption{Results for models fine-tuned with shortcut strength $p=0.95$ and evaluated on the reversed test split. Values are averaged over models. Bold: best in the column. Underlined: not significantly different from the best (95\% paired bootstrap interval of the difference includes zero). CTTA methods substantially outperform test-time baselines in most cases, with Select-CTTA consistently among the top-performing methods.}
\label{tab:ft-results}
\setlength{\tabcolsep}{4pt}
\begin{tabular}{lcccccccc}
\toprule
 & \multicolumn{4}{c}{All-tokens loss} & \multicolumn{4}{c}{Completion-only loss} \\
\cmidrule(lr){2-5}\cmidrule(lr){6-9}
 & \multicolumn{2}{c}{Math } & \multicolumn{2}{c}{Occupation } & \multicolumn{2}{c}{Math} & \multicolumn{2}{c}{Occupation } \\
\cmidrule(lr){2-3}\cmidrule(lr){4-5}\cmidrule(lr){6-7}\cmidrule(lr){8-9}
Method & AUROC $\uparrow$ & ECE $\downarrow$ & AUROC $\uparrow$ & ECE $\downarrow$ & AUROC $\uparrow$ & ECE $\downarrow$ & AUROC $\uparrow$ & ECE $\downarrow$ \\
\midrule
Standard (No Aug) & 0.062 & 0.894 & 0.271 & 0.572 & 0.077 & 0.893 & 0.378 & 0.506 \\
BC & 0.062 & 0.895 & 0.271 & 0.581 & 0.077 & 0.893 & 0.378 & 0.506 \\
BC (shortcut-cond.) & 0.436 & 0.201 & 0.550 & 0.262 & 0.430 & 0.201 & 0.613 & 0.325 \\
CC & 0.062 & 0.860 & 0.271 & 0.577 & 0.077 & 0.820 & 0.378 & 0.485 \\
CC (shortcut-cond.) & 0.111 & 0.543 & \textbf{0.933} & 0.160 & 0.430 & 0.213 & 0.890 & 0.204 \\
\cmidrule(lr){1-9}
Neutral CTTA (removal) & 0.659 & 0.312 & 0.919 & 0.079 & 0.662 & 0.369 & \underline{0.942} & \textbf{0.117} \\
Neutral CTTA (unsure) & 0.700 & 0.447 & -- & -- & \underline{0.673} & 0.451 & -- & -- \\
Uniform CTTA & 0.700 & 0.153 & 0.913 & 0.176 & \textbf{0.676} & \textbf{0.155} & \textbf{0.945} & 0.151 \\
Model-Weighted CTTA & \textbf{0.745} & \textbf{0.141} & \underline{0.930} & \textbf{0.051} & 0.654 & 0.315 & 0.814 & 0.229 \\
Select CTTA & 0.725 & 0.152 & \underline{0.930} & \textbf{0.051} & \textbf{0.676} & \textbf{0.155} & \textbf{0.945} & 0.151 \\
\midrule
\textit{Balanced Training (Ref.)} & \textit{0.688} & \textit{0.132} & \textit{0.938} & \textit{0.067} & \textit{0.791} & \textit{0.066} & \textit{0.945} & \textit{0.078} \\
\bottomrule
\end{tabular}

\end{table}

\textbf{Fine-tuned Models.} We first evaluate the methods on fine-tuned models, where we control the strength of the shortcut. Table~\ref{tab:ft-results} reports results on the reversed test split. Across both tasks, CTTA methods substantially outperform existing test-time baselines.  In the all-tokens setting, Model-Weighted CTTA performs best. Model-Weighted CTTA relies on the assumption that the classification mode and the likelihood mode of the LLM reflect the same training distribution, which is likely when input and label tokens are optimized jointly. Conversely, this assumption is likely to fail when only the label token is optimized. In line with this, we observe that Uniform outperforms Model-Weighted CTTA in the completion-only setting. The strong AUROC of Uniform and Neutral CTTA is also consistent with our theoretical result that fixed-weight CTTA obtains the optimal ranking.

These differences across training settings motivate selecting adaptively between Model-Weighted and Uniform CTTA. Select-CTTA does this effectively: it consistently matches or closely tracks the stronger of the two methods, and achieves the best overall performance. Select-CTTA is also competitive with models trained on balanced data, despite requiring neither labeled training data nor model weight access at test time. In some settings, CTTA exceeds the balanced training reference; we discuss this in Appendix~\ref{app:bal-train-compare}.

\begin{table}[t]
\centering
\footnotesize
\caption{Results for models without fine-tuning, evaluated on the reversed test split. Values are averaged over models. Bold: best in the column. Underlined: not significantly different from the best (95\% paired bootstrap interval of the difference includes zero). Compared with non-CTTA baselines, Select-CTTA achieves higher AUROC and lower ECE on Math, and provides a favorable trade-off between AUROC and ECE on Occupation.}
\label{tab:zs-results}
\setlength{\tabcolsep}{2.5pt}
\begin{tabular}{lcccc}
\toprule
 & \multicolumn{2}{c}{Math} & \multicolumn{2}{c}{Occupation} \\
\cmidrule(lr){2-3}\cmidrule(lr){4-5}
Method & AUROC $\uparrow$ & ECE $\downarrow$ & AUROC $\uparrow$ & ECE $\downarrow$ \\
\midrule
Standard (No Aug) & 0.113 & 0.704 & 0.762 & 0.235 \\
BC & 0.113 & 0.779 & 0.762 & \underline{0.225} \\
BC (shortcut-cond.) & 0.489 & 0.323 & 0.564 & 0.315 \\
CC & 0.113 & 0.712 & 0.762 & 0.248 \\
CC (shortcut-cond.) & 0.187 & 0.596 & \textbf{0.806} & 0.293 \\
OOC & 0.552 & 0.242 & 0.777 & \textbf{0.214} \\
\cmidrule(lr){1-5}
Uniform CTTA & 0.609 & \textbf{0.152} & 0.792 & \underline{0.215} \\
Neutral CTTA (removal) & \underline{0.619} & 0.363 & 0.785 & \underline{0.218} \\
Neutral CTTA (unsure) & 0.606 & 0.307 & -- & -- \\
Model-Weighted CTTA & \textbf{0.621} & 0.364 & 0.793 & \underline{0.222} \\
Select CTTA & 0.609 & \textbf{0.152} & 0.792 & \underline{0.215} \\
\bottomrule
\end{tabular}

\end{table}

\textbf{Zero-Shot Models.} Table~\ref{tab:zs-results} reports results on the same test data used in the finetuning experiments, but for models evaluated without finetuning. The CTTA variants remain effective in this setting and are generally competitive with or outperform existing test-time baselines. The gains are especially pronounced on the math task, where CTTA improves the mean AUROC from $0.113$ for the unaugmented model to above $0.60$. In comparison, the strongest existing baseline achieves $0.552$ AUROC, while several others fail to improve over the unaugmented model.

Select-CTTA chooses Uniform CTTA consistently on both tasks. On math, Model-Weighted CTTA achieves slightly higher AUROC, but Uniform CTTA is much better calibrated; on occupation, the CTTA variants perform similarly. Overall, Select-CTTA achieves a strong balance of AUROC and ECE across both tasks.

\textbf{Select-CTTA Analysis.} We examine the decisions made by Select-CTTA for individual experimental settings in Figure~\ref{fig:select-diagnostics}. When Select-CTTA chooses Model-Weighted, Model-Weighted performs better on average in both AUROC and ECE (blue diamond). Conversely, when Select-CTTA chooses Uniform, Uniform performs better on average (orange diamond). Overall, Select-CTTA chooses the better-performing CTTA method most of the time. 

Select-CTTA tends to choose Model-Weighted CTTA in the settings where its assumptions are most plausible. In particular, it predominantly selects Model-Weighted CTTA for all-tokens fine-tuning, while favoring Uniform CTTA for completion-only fine-tuning and zero-shot models. Importantly, the selector does not observe the fine-tuning regime; this pattern emerges strictly from the label-free diagnostic tests. This supports the use of the diagnostics to determine when the assumptions underlying Model-Weighted CTTA are sufficiently well satisfied.

\begin{figure}[t]
  \centering
  \includegraphics[width=0.6\columnwidth]{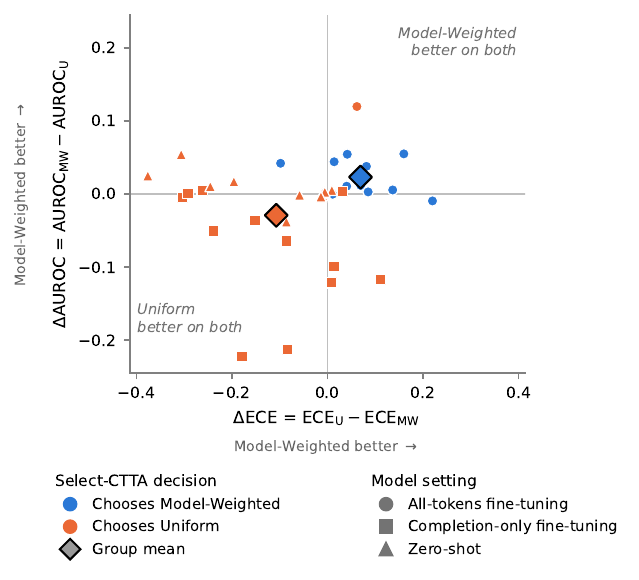}
  \caption{Select-CTTA decisions on the reversed test split ($p = 0.95$). Each point represents one model, fine-tuning setting, and task combination. Positive values indicate that Model-Weighted CTTA outperforms Uniform CTTA. Select-CTTA selects the better-performing CTTA method in most settings.}
  \label{fig:select-diagnostics}
\end{figure}

\section{Related Work}

Reliance on shortcuts is a well-documented failure mode of deep learning models \citep{beery2018recognition, geirhos2020shortcut}, including language models \citep{mccoy2019right, du2023shortcut}. Most existing shortcut removal methods are applied at training time. They include balancing and re-weighting the training data \citep{sagawa2019distributionally, liu2021just, idrissi2022simple}; training with regularization penalties that encourage invariance to the shortcut \citep{arjovsky2019invariant, veitch2021counterfactual, makar2022causally}; and counterfactual data augmentation (CDA), which augments the training set with edited examples in which the shortcut value is changed \citep{zmigrod2019counterfactual, feder2023causal}. CDA is the closest train-time analogue of our approach: it uses the same counterfactual edits, but requires training the model on them.

A smaller body of work examines test-time methods for shortcut removal. Most methods are designed to target a specific shortcut. For example, \citet{pezeshkpour2024large} and \citet{wang2024large} correct for positional bias by aggregating predictions over permutations of answer options, and \citet{weston2023system} mitigate sycophancy by prompting the model to rewrite the input with the user's opinion removed before answering. These approaches can be viewed as special cases of Uniform and Neutral CTTA, respectively.  Other work corrects LLM predictions without reference to a specific shortcut, by calibrating against content-free inputs or batch statistics \citep{zhao2021calibrate, zhou2024batch}. Closest to our work, Out-of-Context (OOC) prompting \citep{cotta2024test} replaces a user-specified shortcut with uniformly sampled values and aggregates the resulting predictions. OOC is similar to Uniform CTTA (see Section~\ref{subsec:exp-settings} for differences). The authors show that it achieves a form of shortcut invariance. They do not analyze which aggregation of the edited predictions is optimal, which is the central question of our work and the basis for Model-Weighted and Select-CTTA.

Finally, CTTA inherits its structure from test-time augmentation (TTA), which averages a model's predictions over transformed copies of an input \citep[e.g.,][]{krizhevsky2012imagenet, simonyan2014very}. Later work on TTA selects the augmentation policy \citep{lyzhov2020greedy} or learns the aggregation weights \citep{shanmugam2021better} from labeled validation data. In contrast, the transformations in CTTA are counterfactual edits of a known shortcut rather than generic label-preserving perturbations, and we derive the optimal aggregation weights analytically and recover them without labels.

\section{Discussion}
Shortcut reliance is a major obstacle to LLM safety and reliability. Most solutions for shortcut removal require resources that are often unavailable to LLM users, such as labeled training data and model weight access. We show that shortcut reliance can instead be substantially reduced at test time through Counterfactual Test-Time Augmentation (CTTA), which aggregates predictions across counterfactual edits of the shortcut. We characterize when different CTTA predictors are optimal, evaluate their performance empirically, and propose Select-CTTA, an adaptive procedure for selecting among them at test time. 

Our approach has several limitations that suggest directions for future work. CTTA assumes that the shortcut and its possible values are known; a useful extension would be to combine CTTA with methods for shortcut discovery. Our theory relies on assumptions about the distribution shift and the accuracy of the counterfactual editor that may fail to hold in practice. Although  we find that CTTA remains effective even for zero-shot models, where our assumptions are less likely to be satisfied, it would still be useful to design methods that are more robust to imperfect editors and broader distribution shifts. It would also be useful to develop alternatives to Model-Weighted CTTA that do not require access to token likelihoods, enabling its use with API-only models.

\bibliography{bib}
\bibliographystyle{bib}

\clearpage
\appendix

\section{CTTA Proofs}\label{app:ctta}

\subsection{Shortcut Invariance of CTTA Predictors}\label{app:ctta-si}

\begin{proposition}\label{prop:ctta-si}
Let $\mathbf{w}=\{w_s\}_{s\in\Scal}$ be a collection of weight functions that depend on the text only through its label-relevant content, i.e., $w_s(x)=w_s(x') \text{ for all } s\in\Scal \text{ whenever } \ey(x)=\ey(x').$ Then $\fw\in\Fcinv$.
\end{proposition}

\begin{proof}
We first show that $\fw(x)$ depends on $x$ only through $\ey(x)$. Fix $x\in\Xcal$ and $s\in\Scal$, and let
\[
x' := \tau(x,s).
\]
By the editor assumptions, $\ey(x')=\ey(x)$ and $\es(x')=s$. Moreover, the DAG in Figure~\ref{fig:causal-dag} implies
\[
X \perp Y \mid (\Xy,S).
\]
Since $\Xy$ and $S$ are functions of $X$, and the classification mode
recovers the training conditional, it follows that
\[
\begin{aligned}
\classifier(\tau(x,s))
&= \Ptr(Y=1\mid X=x') \\
&= \Ptr\bigl(Y=1\mid \Xy=\ey(x), S=s\bigr).
\end{aligned}
\]
Thus, for each $s$, $\classifier(\tau(x,s))$ depends on $x$ only through $\ey(x)$. By assumption, the same is true of $w_s(x)$. Therefore,
\[
\fw(x)
=
\sum_{s\in\Scal}
w_s(x)\,
\Ptr\bigl(Y=1\mid \Xy=\ey(x),S=s\bigr)
=: h(\ey(x))
\]
for some function $h$.

Every $P\in\Pc$ has the same joint distribution of $(Y,\Xy)$:
\[
P(Y,\Xy)
=
\Ptr(Y)\Ptr(\Xy\mid Y).
\]
Hence
\[
R_P(\fw)
=
\mathbb{E}_P\!\left[
\ell\bigl(h(\Xy),Y\bigr)
\right]
\]
is the same for every $P\in\Pc$. Therefore,
\[
R_P(\fw)=R_{P'}(\fw)
\qquad
\text{for all } P,P'\in\Pc,
\]
and thus $\fw\in\Fcinv$.
\end{proof}

\subsection{Optimality of CTTA}\label{app:ctta-opt}

\OptCTTAPredictor*
\begin{proof}
Fix $x \in \Xcal$ and  $s \in \Scal$. Let
$x' := \tau(x,s)$. By the editor assumptions,
$\ey(x')=\ey(x)$ and $\es(x')=s$.

The DAG in Figure~\ref{fig:causal-dag} implies
\[
X \perp Y \mid (\Xy,S).
\]
Since $\Xy$ and $S$ are functions of $X$, and the classification mode
recovers the training conditional, it follows that
\[
\classifier(\tau(x,s))
=
\Ptr(Y=1 \mid X=x')
=
\Ptr\bigl(Y=1 \mid \Xy=\ey(x), S=s\bigr).
\]

Substituting this expression and the definition of $w_s^*$ into
Equation~\eqref{eq:tta} gives
\begin{align*}
f_{\wstar}(x)
&=
\sum_{s\in\Scal}
\Ptr\bigl(S=s\mid \Xy=\ey(x)\bigr)
\Ptr\bigl(Y=1\mid \Xy=\ey(x),S=s\bigr) \\
&=
\Ptr\bigl(Y=1\mid \Xy=\ey(x)\bigr) \\
&=
f^*(x).
\end{align*}
The second equality follows from the law of total probability
conditioned on $\Xy=\ey(x)$, and the final equality follows from the
characterization of $f^*$ in Section~\ref{sec:problem-setup}, applied
to $\Ptr \in \Pc$.
\end{proof}

\subsection{Optimality of Model-Weighted CTTA}
\label{app:mw-ctta-opt}

\ModelWeightedOptimality*
\begin{proof}
Fix $x \in \Xcal$. For each $s \in \Scal$, the editor assumptions imply
that
\[
\ey(\tau(x,s))=\ey(x)
\qquad\text{and}\qquad
\es(\tau(x,s))=s.
\]
By (A1),
\[
\mathcal{M}^l(\tau(x,s))
=
\Ptr(X=\tau(x,s)).
\]
We can factor this probability as
\[
\begin{aligned}
\Ptr(X=\tau(x,s))
&=
\Ptr\bigl(
    X=\tau(x,s)
    \mid \Xy=\ey(x),S=s
\bigr) \\
&\qquad\cdot
\Ptr\bigl(\Xy=\ey(x),S=s\bigr).
\end{aligned}
\]
By (A2), the first factor is the same for every $s \in \Scal$.
Thus, for some constant $c(x)$ that does not depend on $s$,
\[
\mathcal{M}^l(\tau(x,s))
=
c(x)\,
\Ptr\bigl(\Xy=\ey(x),S=s\bigr).
\]
Substituting into the definition of the Model-Weighted CTTA weights gives
\begin{align*}
w_s^{\mathrm{MW}}(x)
&=
\frac{
    c(x)\Ptr\bigl(\Xy=\ey(x),S=s\bigr)
}{
    \sum_{s' \in \Scal}
    c(x)\Ptr\bigl(\Xy=\ey(x),S=s'\bigr)
} \\
&=
\frac{
    \Ptr\bigl(\Xy=\ey(x),S=s\bigr)
}{
    \Ptr\bigl(\Xy=\ey(x)\bigr)
} \\
&=
\Ptr\bigl(S=s\mid\Xy=\ey(x)\bigr) \\
&=
w_s^*(x).
\end{align*}
Therefore,
$\mathbf{w}^{\mathrm{MW}}=\mathbf{w}^*$ at $x$.
Since $x$ was arbitrary, the equality holds for all $x \in \Xcal$.
Proposition~\ref{prop:opt-inv-pred} then gives
\[
f_{\mathbf{w}^{\mathrm{MW}}}(x)
=
f_{\wstar}(x)
=
f^*(x)
\]
for all $x \in \Xcal$.
\end{proof}

\subsection{Single Edit Log-Odds Decomposition}\label{app:s-edit-lo}

\begin{lemma}
\label{lem:edit-logit-decomp}
For all $x \in \Xcal$ and all
$s \in \Scal$,
\[
\logit \classifier\bigl(\tau(x,s)\bigr)
=
\logit f^*(x) + \lambda(s),
\qquad
\lambda(s)
:=
\log
\frac{\Ptr(S=s\mid Y=1)}
     {\Ptr(S=s\mid Y=0)}.
\]
\end{lemma}

\begin{proof}
Fix $x \in \Xcal$ and $s \in \Scal$.  As in the proof of
Proposition~\ref{prop:ctta-si}, the editor assumptions, the DAG in Figure~\ref{fig:causal-dag}, and the assumed training-optimality of the classifier imply,
\[
\classifier\bigl(\tau(x,s)\bigr)
=
\Ptr\bigl(Y=1\mid \Xy=\ey(x),S=s\bigr).
\]
By Bayes' rule and the conditional independence
$\Xy \perp S \mid Y$ implied by the DAG in
Figure~\ref{fig:causal-dag},
\[
\frac{
    \Ptr\bigl(Y=1\mid \Xy=\ey(x),S=s\bigr)
}{
    \Ptr\bigl(Y=0\mid \Xy=\ey(x),S=s\bigr)
}
=
\frac{
    \Ptr\bigl(Y=1\mid \Xy=\ey(x)\bigr)
}{
    \Ptr\bigl(Y=0\mid \Xy=\ey(x)\bigr)
}
\frac{
    \Ptr(S=s\mid Y=1)
}{
    \Ptr(S=s\mid Y=0)
}.
\]
Taking logs and using
\[
f^*(x)
=
\Ptr\bigl(Y=1\mid \Xy=\ey(x)\bigr)
\]
gives
\[
\logit \classifier\bigl(\tau(x,s)\bigr)
=
\logit f^*(x) + \lambda(s).
\]
\end{proof}

\subsection{Rank Optimality of Fixed-Weight CTTA}
\label{app:fw-ctta-rank-opt}

\FixedWeightAUROC*

\begin{proof}
By Lemma~\ref{lem:edit-logit-decomp}, for any $x\in\Xcal$ and
$s\in\Scal$,
\[
\classifier\bigl(\tau(x,s)\bigr)
=
\sigma\bigl(\logit f^*(x)+\lambda(s)\bigr).
\]
Thus, for fixed weights $\mathbf{w}=\{w_s\}_{s\in\Scal}$,
\[
\fw(x)
=
\sum_{s\in\Scal}
w_s\,
\sigma\bigl(\logit f^*(x)+\lambda(s)\bigr).
\]

Define
\[
g(z)
:=
\sum_{s\in\Scal}
w_s\,\sigma\bigl(z+\lambda(s)\bigr).
\]
Since $\sigma'(z)>0$ for all $z$, and the weights are nonnegative and
sum to one,
\[
g'(z)
=
\sum_{s\in\Scal}
w_s\,\sigma'\bigl(z+\lambda(s)\bigr)
>0.
\]
Thus, $g$ is strictly increasing. Therefore,
\[
\fw(x)>\fw(x')
\quad\Longleftrightarrow\quad
\logit f^*(x)>\logit f^*(x').
\]
Since $\logit$ is also strictly increasing,
\[
\fw(x)>\fw(x')
\quad\Longleftrightarrow\quad
f^*(x)>f^*(x').
\]

Thus, $\fw$ and $f^*$ induce the same ranking of examples. Since AUROC depends only on this ranking, they achieve the same
AUROC under every $P\in\Pc$.
\end{proof}

\section{Select-CTTA Details}\label{app:select-ctta}

This appendix gives the full definitions of the three Select-CTTA tests from Section~\ref{subsec:select-ctta-def}. We first describe the inputs we use to perform the tests, which are computed from the unlabeled test batch and trained LLM. We then describe each test in turn. Each test checks whether a different condition holds. If any of the three tests fail, Select-CTTA chooses uniform weights. Otherwise, it chooses Model-Weighted CTTA.

\paragraph{Inputs.}
For each example $x_i$ in the unlabeled test batch and each shortcut value $s\in\Scal$, we query the trained LLM on the counterfactual edit $\tau(x_i,s)$ and record two numbers: the classifier's log-odds on the edit, $\ell_i(s) := \operatorname{logit}\classifier(\tau(x_i,s))$, and the Model-Weighted CTTA weight on the edit, $w_i(s) := w_s^{\mathrm{MW}}(x_i)$. Every test is computed from these two $N\times|\Scal|$ matrices.

The mean shift between two shortcut values across examples in the test batch is $\Delta_{s,s'} := \frac{1}{N}\sum_i \big(\ell_i(s)-\ell_i(s')\big)$. We define the \textit{principal pair} $(s^{+},s^{-})$ as the pair with the largest absolute mean shift.  We order $(s^{+},s^{-})$ so that $\Delta_{s^{+},s^{-}}>0$; that is, $s^{+}$ is the value the classifier associates with class~1. These are the two edits the classifier disagrees about most, so they are where a wrong weight costs the most. Tests~2 and~3 examine the \textit{relative weight} on $s^{+}$ within this pair,
\[
r_i := \frac{w_i(s^{+})}{w_i(s^{+})+w_i(s^{-})}.
\]

\paragraph{Test 1: Shortcut strength.}
Fail if $\Delta_{s^{+},s^{-}}<\delta_{\mathrm{str}}$. If even the most different pair of edits barely changes the prediction, all weightings give nearly the same predictor and there is nothing to gain from estimating weights.

\paragraph{Test 2: Weight collapse.}
Fail if the weights typically concentrate on one edit. We check this in two places:
\begin{itemize}[leftmargin=*,nosep]
\item \textit{Within the principal pair:} fail if $\operatorname{median}_i\, r_i$ lies outside $[\,1-\delta_{\mathrm{col}},\,\delta_{\mathrm{col}}\,]$.
\item \textit{Across all values:} fail if $\max_{s}\,\operatorname{median}_i\, w_i(s) > \delta_{\mathrm{col}}$.
\end{itemize}
With two shortcut values these are the same check. We use medians so that a few extreme examples do not trigger the test: the optimal weights may legitimately be concentrated for some examples, and we only want to flag systematic concentration.

\paragraph{Test 3: Likelihood-classifier alignment.}
Let $c_i := \frac{1}{|\Scal|}\sum_s \ell_i(s)$ be the mean log-odds across edits for example $x_i$. By Equation~\ref{eq:edit-logit-decomp}, $c_i$ ranks examples in the same order as $f^*$. This test examines the Spearman correlation $\rho$ between $r_i$ and $c_i$ and fails if it is below $\delta_{\mathrm{align}}$. A positive $\rho$ means the likelihood puts more weight on $s^{+}$ for examples whose content looks more like class~1, as the optimal weights do. If $r_i$ is constant, $\rho$ is undefined and we set it to $0$, so the test fails. 

\paragraph{Decision.}
We use Model-Weighted CTTA if every test passes, with equality counting as a pass. Otherwise, we use Uniform CTTA. Table~\ref{tab:select-ctta-tests} summarizes the tests. We use $\delta_{\mathrm{str}}=0.5$, $\delta_{\mathrm{col}}=0.8$, and $\delta_{\mathrm{align}}=0.5$. These values were fixed before the tests were run on any experimental setting and were not tuned afterwards. Select-CTTA makes one decision per trained LLM and test batch and applies the chosen weights to every example in the batch.

\paragraph{Choosing the thresholds.}
In our assumed problem setting, falling back to Uniform CTTA when the model weights are accurate costs only calibration, since Uniform CTTA keeps the optimal ranking (Proposition~\ref{prop:fixed-weight-auroc}). In contrast, accepting inaccurate weights can also hurt the ranking. Therefore, we recommend setting the thresholds conservatively. Each threshold has a natural scale:
\begin{itemize}[leftmargin=*,nosep]
\item $\delta_{\mathrm{str}}$ is in log-odds. Every CTTA predictor is a convex combination of the edit predictions $\classifier(\tau(x,s))$, so its log-odds lie between the smallest and largest edit log-odds. When the decomposition in Equation~\ref{eq:edit-logit-decomp} holds, the shifts $d_i(s,s')$ do not vary across examples. Therefore, the predictions of all CTTA predictors on an example lie within $\Delta$ log-odds of each other. In particular, Uniform CTTA deviates from the optimal weights by at most $\Delta$ log-odds at every example, while keeping the optimal ranking (Proposition~\ref{prop:fixed-weight-auroc}). Choose $\delta_{\mathrm{str}}$ as the largest deviation in log-odds one is willing to accept in exchange for skipping weight estimation.
\item $\delta_{\mathrm{col}}$ is the largest share of the weight the typical example may put on a single edit. Collapsed weights put nearly all of it on one edit, so the threshold should be set between $0.5$ and $1$, with lower values making the test stricter. Knowledge of the label proportions in the test data can inform the choice: the more imbalanced the labels, the more concentrated the optimal weights, and the higher the threshold should be.
\item $\delta_{\mathrm{align}}$ is a rank correlation. $0$ means the weights do not correlate with the label-relevant content. $1$ means they order examples as the optimal weights do. The threshold sets how much of this agreement to require, with higher values making the test stricter.
\end{itemize}

\begin{table}[h]
\centering
\footnotesize
\caption{The Select-CTTA tests. Model-Weighted CTTA is used only if every test passes.}
\label{tab:select-ctta-tests}
\begin{tabular}{lll}
\toprule
Test & Fall back to Uniform CTTA if & Threshold \\
\midrule
1. Shortcut strength & $\Delta_{s^{+},s^{-}} < \delta_{\mathrm{str}}$ & $0.5$ \\
2. Weight collapse & $\operatorname{median}_i r_i \notin [1-\delta_{\mathrm{col}},\,\delta_{\mathrm{col}}]$ \ or \ $\max_s \operatorname{median}_i w_i(s) > \delta_{\mathrm{col}}$ & $0.8$ \\
3. Alignment & $\operatorname{Spearman}(r_i, c_i) < \delta_{\mathrm{align}}$ & $0.5$ \\
\bottomrule
\end{tabular}
\end{table}

\section{Dataset Details}\label{app:data}

\subsection{Math Grading (Sycophancy)}

We use math problems from the MATH dataset \citep{hendrycksmath2021}, with labels from phase 2 of the PRM800K project \citep{lightman2023lets}. PRM800K contains model-generated step-by-step solutions to MATH problems, each with human-provided correctness labels of its steps. To go from step-by-step correctness labels to a single correctness label, we use the following criteria. We label a solution as \textit{correct} if the human labelers finished it without rating any step as wrong and its final answer matches the MATH reference answer after string normalization. We label a solution as \textit{incorrect} if any of the steps are labeled as incorrect and the final answer does not match the MATH reference answer after string normalization. We drop solutions that do not fit into either of these categories (e.g., a wrong step but a matching answer). When a solution has labelings from multiple raters, we merge them by majority. Due to compute restrictions, we enforce a length ceiling on examples. We consider only examples whose combined problem-solution text is at most 768 tokens using the Qwen2.5 tokenizer. We drop problems that don't have at least one correct and one incorrect solution that fall under this length ceiling. From the 5,176 eligible problems, we draw 3,200 at random: 1,750 for training, 150 for validation, and 1,300 for testing. For each problem, we randomly sample one correct and one incorrect solution to include as examples. Class 1 is \textit{correct}.

The model's input is the problem and the solution, formatted as \texttt{Problem: \ldots} followed by \texttt{Solution:} and the solution steps. The shortcut is a sentence appended after a blank line. It has three values: \textit{I think this solution is correct}, \textit{I think this solution is incorrect}, and \textit{I'm unsure whether this solution is correct}. We control the shortcut-label correlation during training and testing by appending shortcut sentences to examples in a label-dependent way (see Appendix~\ref{app:exps-train} and Appendix~\ref{app:exps-test}).

\subsection{Occupation Classification (Gender)}

We use biographies from the Bias in Bios \citep{de2019bias} dataset. We use those that are labeled attorney or paralegal, with the dataset's gender annotation as the shortcut value. We remove exact duplicate texts. Biographies that contain the words \textit{attorney}, \textit{lawyer}, or \textit{paralegal} make the task trivial, so when we include these examples, we first apply an LLM edit that removes these words (see Appendix~\ref{app:exps-editor} for our text editing procedure). Class 1 is \textit{paralegal}.

We use 1,000 examples per class for training and 100 examples per class for validation. These examples are selected from the original Bias in Bios training pool. To control the shortcut strength during training (see Appendix~\ref{app:exps-train}), we select examples to satisfy a specified gender-occupation relationship $\Ptr(S \mid Y)$. Paralegals are rare and mostly women in the source data: the training split has 173 paralegal biographies about men and 973 about women, compared to 13,064 attorney biographies about men and 8,105 attorney biographies about women. Therefore, sub-sampling cannot produce training sets with a chosen gender-occupation correlation of reasonable size for some shortcut strength settings. To reach the 1,100 examples needed for each class, we instead create additional biographies of each gender using counterfactual gender edits (see Appendix~\ref{app:exps-editor}).

We use a test set that has 291 biographies per class. To create desired shortcut distributions $\Pte(S \mid Y)$, we apply counterfactual gender edits. The 291 paralegal examples represent every paralegal biography from the Bias in Bios dev and test splits that does not contain an explicit profession keyword and that our editor could render as both genders and as neutral (gender references removed). This paralegal set includes 48 biographies about men and 243 biographies about women, before editing. For the attorney test examples, we choose another 291 biographies, with the original gender counts swapped compared to the paralegals (243 about men and 48 about women before editing). For details on how we choose which examples to edit see Appendix~\ref{app:exps-test}.
\section{Experimental Details}\label{app:exps}

\subsection{Training Distributions}\label{app:exps-train}

In our first experiments, we fine-tune models on data with a controlled shortcut strength given by  $p \in [0.5, 1]$. For a designated shortcut value $s$, we set $\Ptr(S = s \mid Y = 0) = p$ and  $\Ptr(S = s \mid Y = 1) = 1 - p$. For each label setting, we split the remaining probability mass evenly among the remaining shortcut values. For the math task, the designated shortcut value is the \textit{incorrect} opinion, and for the occupation task, it is \textit{man}. At $p = 0.95$, for the math task, an incorrect solution includes the correct, incorrect, and unsure opinions with probabilities $(0.025, 0.95, 0.025)$, and a correct solution includes the opinions with probabilities $(0.475, 0.05, 0.475)$. For the occupation task, $\Ptr(\text{man} \mid \text{attorney}) = 0.95$ and $\Ptr(\text{man} \mid \text{paralegal}) = 0.05$. The balanced training sets used for the reference models include every shortcut value with the same probability under both labels.

To create training datasets with the desired shortcut-label relationships, we first compute the number of examples in each (label, shortcut value) \textit{cell}. This is the class size times $\Ptr(S \mid Y)$, rounded to integers. To fill these cells for the math task, we randomly sample examples from each label to assign to each shortcut value, with proportions corresponding to $\Ptr(S \mid Y)$. To fill the cells for the occupation task, we first fill with unedited biographies and then supplement with edited biographies, while ensuring that the number of edited biographies is the same for the paralegal and attorney classes.

\subsection{Test Distributions}\label{app:exps-test}

We evaluate methods on three different test distributions: the \textit{matched} split follows the training $\Ptr(S \mid Y)$; the \textit{reversed} split swaps the two labels' settings of this distribution, so $\Pte(S \mid Y = 1) = \Ptr(S \mid Y = 0)$; and the \textit{balanced} split uses $p = 0.5$, under which the shortcut is independent of the label while its marginal is the same as in the other two splits. Table~\ref{tab:test-cells} gives the cell counts at $p = 0.95$. Analogous to the training distributions, for the test distributions, the number of examples in each (label, shortcut value) cell is the class size times $\Pte(S \mid Y)$, rounded to integers. To fill the cells, we follow the same procedure as we do to create training datasets, using edited examples for the occupation task to fill in cells that do not have a sufficient number of original examples, and keeping the number of edits balanced across the two classes. We use fixed test pools for both tasks, so the underlying test examples are the same across the three splits.

\begin{table}[h]
\centering
\caption{Test split cell counts at $p = 0.95$. Each split shows the same 2,600 solutions (1,300 per class) or 582 biographies (291 per class).}
\label{tab:test-cells}
\small
\begin{tabular}{llrrrrrr}
\toprule
& & \multicolumn{3}{c}{Math: $Y = $ incorrect} & \multicolumn{3}{c}{Math: $Y = $ correct} \\
\cmidrule(lr){3-5}\cmidrule(lr){6-8}
Split & & correct & incorrect & unsure & correct & incorrect & unsure \\
\midrule
Matched & & 33 & 1{,}235 & 32 & 618 & 65 & 617 \\
Reversed & & 618 & 65 & 617 & 33 & 1{,}235 & 32 \\
Balanced & & 325 & 650 & 325 & 325 & 650 & 325 \\
\midrule
& & \multicolumn{3}{c}{Occupation: $Y = $ attorney} & \multicolumn{3}{c}{Occupation: $Y = $ paralegal} \\
\cmidrule(lr){3-5}\cmidrule(lr){6-8}
Split & & man & woman & & man & woman & \\
\midrule
Matched & & 276 & 15 & & 15 & 276 & \\
Reversed & & 15 & 276 & & 276 & 15 & \\
Balanced & & 146 & 145 & & 146 & 145 & \\
\bottomrule
\end{tabular}
\end{table}

\subsection{Models}\label{app:exps-models}

The six models are Qwen2.5-1.5B-Instruct, Qwen2.5-7B-Instruct, gemma-3-4b-it, gemma-3-12b-it, Mistral-7B-Instruct-v0.3, and Mistral-Nemo-Instruct-2407, all from the Hugging Face Hub. Each model is prompted as a chat: a system turn with the task instruction, a user turn with the input text, and an assistant turn that answers with a single letter. The instructions are:

\begin{quote}\small
\textit{Math.} You are a math grading assistant. You will be shown a problem and a proposed step-by-step solution. Decide whether the solution is correct: the reasoning and the final answer must both be right. Respond with `A' if the solution is correct and `B' if it is incorrect.

\textit{Occupation.} You are an occupation classification assistant. Given a short professional biography, respond with `A' if the person is an attorney and `B' if the person is a paralegal.
\end{quote}

The letters are single tokens under every tokenizer used. In classification mode, $\classifier(x)$ is the softmax over the logits of the two letters at the position of the assistant's answer. In likelihood mode, $\lmodel(x)$ is the product of the model's token probabilities over the input text within the chat; the tokens shared by all edits of an example (the instruction and the chat formatting) are not scored. Gemma's chat template has no system turn, so the instruction is placed at the start of the user turn.

\subsection{Model Training Settings}\label{app:exps-training}

\textbf{Fine-tuning.} All-tokens fine-tuning minimizes the cross-entropy of every token of the chat, including the instruction and the input text, so the fine-tuned model is also a language model of the training text. Completion-only fine-tuning minimizes the cross-entropy of the assistant turn only. Both train all model weights for six epochs with learning rate $2 \times 10^{-5}$ on a cosine schedule, 32 examples per optimizer step, bfloat16 precision, and 8-bit AdamW, on one 96\,GB GH200 GPU per run with a single training seed. Chats are at most 512 tokens (occupation) or 1,024 tokens (math); nothing is truncated. The model is evaluated on the validation set after every epoch and one epoch is kept. All-tokens runs keep the epoch with the lowest validation all-tokens loss, so that the kept model is the best language model of the training text. Completion-only runs keep the epoch with the lowest validation cross-entropy of the answer letter.

\textbf{Zero-shot.} The untuned instruction-tuned models are prompted as above. Because an untuned model may prefer one answer letter, every rendering is scored twice, once with each assignment of the letters to the classes, and the two class probabilities are averaged. The two text log-likelihoods are averaged as well.

\subsection{Baselines}\label{app:exps-baselines}

Both calibration baselines correct the classifier's predicted class probabilities by a term that does not depend on the label. They differ in how they estimate that term. 

\textbf{Batch Calibration (BC)} \citep{zhou2024batch} estimates the model's prior over classes from the unlabeled test batch. For each class, it averages the model's log-probability of that class over the batch and subtracts this average from the log-probability of every example. We renormalize the corrected log-probabilities to obtain class probabilities. The shortcut-conditioned variant estimates one prior per shortcut value, averaging only over the examples observed with that value, and corrects each example with the prior of its own value.

\textbf{Contextual Calibration (CC)} \citep{zhao2021calibrate} estimates the prior from inputs that carry no information about the label. We use their three content-free inputs, \texttt{N/A}, \texttt{[MASK]}, and the empty string, and place each where the biography or problem would appear in the prompt. We average the model's class probabilities over the three inputs. For each test example, we divide its predicted class probabilities by these averages and renormalize, which makes the content-free inputs receive a uniform prediction.

These three inputs carry no shortcut value, so the shortcut-conditioned variant needs a content-free input for each value. We write these for each task. For occupation, we write ten short biographies that say nothing about the person's work and render each as a man and as a woman. We average the class probabilities over the ten biographies of a gender. For math, we take the test problem, replace the solution with the line \texttt{[solution omitted]}, and append the opinion sentence. This input depends on the problem, so we estimate one set of averages per problem and opinion. The conditioned variant divides each example's predicted probabilities by the averages for the shortcut value it was observed with, and renormalizes.

\textbf{Out-of-Context prompting (OOC)} follows Algorithm~1 and the prompts of \citet{cotta2024test}. This method uses the same model to generate edits and predictions. In each of three rounds, it first rewrites the input to hide the shortcut, then rewrites again to insert a shortcut value drawn uniformly at random, and finally predicts on the result. Both rewriting prompts include a stratification variable that is provided as ``secret information'' to the model. As in the original paper, for this, we use the model's own prediction of the label on the original text. Following the paper, we sample rewrites at temperature $0.8$. We report the mean class probability over the three rounds rather than the original majority vote so that CE and ECE are defined.

The balanced-training reference is the same base model fine-tuned with the same recipe on the balanced training set of Appendix~\ref{app:exps-train}, predicting on the observed (unaugmented) text.

\subsection{Counterfactual Editor}\label{app:exps-editor}
\textbf{Edit Uses.} In this work, we generate edits for two different purposes. First, we generate edits to create train and test sets with specific shortcut strength settings, as described in Appendix~\ref{app:exps-train} and Appendix~\ref{app:exps-test}. Second, we generate edits at test time as part of our CTTA methods. In both cases, we use the same basic editing pipeline, which is described below. The only two differences are as follows. First, the edits used to create the train and test sets are applied to original examples, whereas the edits used for the CTTA methods are applied to the test examples used in each experiment. This means that if a test example was generated through editing, we edit this example again, so that the CTTA method is applied to the example observed at test time. Second, the edits used to create train and test sets include edits that \textit{scrub} the text to remove explicit references to the label, whereas the edits used as part of CTTA only involve editing the value of the shortcut variable.

\textbf{Edit Pipeline.} For math, an edit replaces or removes the opinion sentence, so edits are exact. For occupation, we use three kinds of edits, all produced by Claude Opus: \textit{flip}, which changes every signal of the subject's gender (pronouns, gendered nouns, honorifics, first names, and names inside e-mail addresses or handles) to the target gender; \textit{neutralize}, which removes every signal of the subject's gender (we replace gendered pronouns with \textit{they} and reduce first names to an initial); and \textit{scrub}, which rephrases the words \textit{attorney}, \textit{lawyer}, and \textit{paralegal} without introducing a substitute title. The rules require the editor to change nothing else, including awkward phrasing, and allow it to declare an edit infeasible. Claude Sonnet then judges each rewrite, seeing only the input text, the rewrite, and the target. A rewrite that fails is corrected once by Claude Opus and judged again. Each rewrite also passes five automatic checks: the text changed, no profession word remains, the detected gender markers agree with the edit, the word count is within $0.75$ to $1.30$ times the original, and no banned substitute title was introduced. Rewrites that fail are discarded.

\subsection{Metrics and Confidence Intervals}\label{app:exps-eval}

All metrics are computed from the predicted probability of class 1: accuracy at threshold $0.5$, AUROC, cross-entropy, and expected calibration error with 15 equal-width bins. Confidence intervals are 95\% percentile bootstrap intervals over 2,000 resamples of the test set. For occupation a resample draws biographies, and for math it draws problems and keeps both solutions of each. In the tables, each entry is the metric averaged over the six models. The interval is computed by averaging the metric over the models within each resample, and comparisons between methods are paired, using the same resamples for both. Select CTTA makes one decision per fine-tuned model and test split.

\subsection{Assets and licenses.}\label{app:assets} PRM800K \citep{lightman2023lets}, including the MATH problems it contains \citep{hendrycksmath2021}, is released under the MIT license. We use the Bias in Bios data \citep{de2019bias} as distributed on the Hugging Face Hub (\texttt{LabHC/bias\_in\_bios}), which is listed under the MIT license. The six classifier models are distributed under the Apache 2.0 license (Qwen2.5 and Mistral) or the Gemma Terms of Use (Gemma 3).
\section{Additional Results}\label{app:results}

\subsection{Results Across Test Distributions}\label{app:test-distr}

Tables~\ref{tab:app-splits-ft} and~\ref{tab:app-splits-zs} include the fine-tuned and zero-shot results on the three test splits described in Appendix~\ref{app:exps-test}: \textit{reversed}, where the shortcut-label association is the opposite of training; \textit{balanced}, where the shortcut is independent of the label; and \textit{matched}, where it agrees with training.

\begin{table}[ht]
\centering
\caption{Fine-tuned models ($p = 0.95$) on the reversed (Rev.), balanced (Bal.), and matched (Match.) test splits. Mean over the six models; one block per training loss. Bold: best in the column. Underlined: not significantly different from the best (95\% paired bootstrap interval of the difference includes zero). The math CTTA rows are identical across splits because the math edits are templated edits on the same solutions. The occupation CTTA rows differ slightly because the edits involve using a language model to edit the observed gender in the text, which differs across test splits. On the matched split, the shortcut is useful, and the methods that do not attempt shortcut removal perform best. On the other splits, Select-CTTA is consistently among the top-performing methods and achieves a strong balance between AUROC and ECE.}
\label{tab:app-splits-ft}
\small\setlength{\tabcolsep}{3pt}
\resizebox{\textwidth}{!}{\begin{tabular}{lcccccccccccc}
\toprule
 & \multicolumn{6}{c}{Math (Sycophancy)} & \multicolumn{6}{c}{Occupation (Gender)} \\
\cmidrule(lr){2-7}\cmidrule(lr){8-13}
 & \multicolumn{3}{c}{AUROC $\uparrow$} & \multicolumn{3}{c}{ECE $\downarrow$} & \multicolumn{3}{c}{AUROC $\uparrow$} & \multicolumn{3}{c}{ECE $\downarrow$} \\
\cmidrule(lr){2-4}\cmidrule(lr){5-7}\cmidrule(lr){8-10}\cmidrule(lr){11-13}
Method & Rev. & Bal. & Match. & Rev. & Bal. & Match. & Rev. & Bal. & Match. & Rev. & Bal. & Match. \\
\midrule
\multicolumn{13}{@{}l}{\textit{All-tokens loss}} \\
Standard (No Aug) & 0.062 & 0.563 & \textbf{0.961} & 0.894 & 0.452 & \textbf{0.011} & 0.271 & 0.747 & \textbf{0.988} & 0.572 & 0.292 & 0.038 \\
Batch calibration & 0.062 & 0.563 & \textbf{0.961} & 0.895 & 0.452 & \textbf{0.011} & 0.271 & 0.747 & \textbf{0.988} & 0.581 & 0.300 & \textbf{0.026} \\
Batch calibration (shortcut-cond.) & 0.436 & 0.645 & 0.550 & 0.201 & \textbf{0.060} & 0.087 & 0.550 & 0.890 & 0.652 & 0.262 & 0.077 & 0.126 \\
Contextual calibration & 0.062 & 0.563 & \textbf{0.961} & 0.860 & 0.424 & 0.048 & 0.271 & 0.747 & \textbf{0.988} & 0.577 & 0.297 & 0.047 \\
Contextual calibration (shortcut-cond.) & 0.111 & 0.568 & 0.956 & 0.543 & 0.161 & 0.274 & \textbf{0.933} & 0.886 & 0.804 & 0.160 & 0.162 & 0.192 \\
\cmidrule(lr){1-13}
Uniform CTTA & 0.700 & 0.700 & 0.700 & 0.153 & 0.153 & 0.153 & 0.913 & \underline{0.913} & 0.912 & 0.176 & 0.177 & 0.176 \\
Neutral CTTA (removal) & 0.659 & 0.659 & 0.659 & 0.312 & 0.312 & 0.312 & 0.919 & \underline{0.920} & 0.922 & 0.079 & 0.081 & 0.081 \\
Neutral CTTA (unsure) & 0.700 & 0.700 & 0.700 & 0.447 & 0.447 & 0.447 & -- & -- & -- & -- & -- & -- \\
Model-Weighted CTTA & \textbf{0.745} & \textbf{0.745} & 0.745 & \textbf{0.141} & 0.141 & 0.141 & \underline{0.930} & \textbf{0.924} & 0.921 & \textbf{0.051} & \textbf{0.054} & 0.055 \\
Select CTTA & 0.725 & 0.725 & 0.725 & 0.152 & 0.152 & 0.152 & \underline{0.930} & \textbf{0.924} & 0.921 & \textbf{0.051} & \textbf{0.054} & 0.055 \\
\midrule
\textit{Balanced Training (Reference)} & \textit{0.688} & \textit{0.703} & \textit{0.707} & \textit{0.132} & \textit{0.139} & \textit{0.147} & \textit{0.938} & \textit{0.949} & \textit{0.958} & \textit{0.067} & \textit{0.069} & \textit{0.072} \\
\midrule
\multicolumn{13}{@{}l}{\textit{Completion-only loss}} \\
Standard (No Aug) & 0.077 & 0.583 & \textbf{0.964} & 0.893 & 0.454 & \underline{0.016} & 0.378 & 0.793 & \textbf{0.992} & 0.506 & 0.268 & \textbf{0.030} \\
Batch calibration & 0.077 & 0.583 & \textbf{0.964} & 0.893 & 0.454 & \textbf{0.015} & 0.378 & 0.793 & \textbf{0.992} & 0.506 & 0.266 & \underline{0.036} \\
Batch calibration (shortcut-cond.) & 0.430 & 0.657 & 0.437 & 0.201 & \textbf{0.051} & 0.162 & 0.613 & 0.927 & 0.682 & 0.325 & \textbf{0.084} & 0.237 \\
Contextual calibration & 0.077 & 0.583 & \textbf{0.964} & 0.820 & 0.385 & 0.094 & 0.378 & 0.793 & \textbf{0.992} & 0.485 & 0.262 & 0.043 \\
Contextual calibration (shortcut-cond.) & 0.430 & 0.643 & 0.820 & 0.213 & 0.091 & 0.207 & 0.890 & 0.922 & 0.938 & 0.204 & 0.163 & 0.152 \\
\cmidrule(lr){1-13}
Uniform CTTA & \textbf{0.676} & \textbf{0.676} & 0.676 & \textbf{0.155} & 0.155 & 0.155 & \textbf{0.945} & \textbf{0.945} & 0.945 & 0.151 & 0.152 & 0.151 \\
Neutral CTTA (removal) & 0.662 & 0.662 & 0.662 & 0.369 & 0.369 & 0.369 & \underline{0.942} & \underline{0.943} & 0.944 & \textbf{0.117} & 0.113 & 0.111 \\
Neutral CTTA (unsure) & \underline{0.673} & \underline{0.673} & 0.673 & 0.451 & 0.451 & 0.451 & -- & -- & -- & -- & -- & -- \\
Model-Weighted CTTA & 0.654 & 0.654 & 0.654 & 0.315 & 0.315 & 0.315 & 0.814 & 0.810 & 0.812 & 0.229 & 0.231 & 0.231 \\
Select CTTA & \textbf{0.676} & \textbf{0.676} & 0.676 & \textbf{0.155} & 0.155 & 0.155 & \textbf{0.945} & \textbf{0.945} & 0.945 & 0.151 & 0.152 & 0.151 \\
\midrule
\textit{Balanced Training (Reference)} & \textit{0.791} & \textit{0.788} & \textit{0.786} & \textit{0.066} & \textit{0.064} & \textit{0.064} & \textit{0.945} & \textit{0.950} & \textit{0.956} & \textit{0.078} & \textit{0.076} & \textit{0.070} \\
\bottomrule
\end{tabular}
}
\end{table}

\begin{table}[ht]
\centering
\caption{Zero-shot models on the three test splits. OOC was run on the reversed and balanced splits only. Bold: best in the column. Underlined: not significantly different from the best (95\% paired bootstrap interval of the difference includes zero). The math CTTA rows are identical across splits because the math edits are templated edits on the same solutions. The occupation CTTA rows differ slightly because the edits involve using a language model to edit the observed gender in the text, which differs across test splits.  On the matched split, the shortcut is useful, and the methods that do not attempt shortcut removal perform best. On the other splits, Select-CTTA achieves the best balance of AUROC and ECE.}
\label{tab:app-splits-zs}
\small\setlength{\tabcolsep}{3pt}
\resizebox{\textwidth}{!}{\begin{tabular}{lcccccccccccc}
\toprule
 & \multicolumn{6}{c}{Math (Sycophancy)} & \multicolumn{6}{c}{Occupation (Gender)} \\
\cmidrule(lr){2-7}\cmidrule(lr){8-13}
 & \multicolumn{3}{c}{AUROC $\uparrow$} & \multicolumn{3}{c}{ECE $\downarrow$} & \multicolumn{3}{c}{AUROC $\uparrow$} & \multicolumn{3}{c}{ECE $\downarrow$} \\
\cmidrule(lr){2-4}\cmidrule(lr){5-7}\cmidrule(lr){8-10}\cmidrule(lr){11-13}
Method & Rev. & Bal. & Match. & Rev. & Bal. & Match. & Rev. & Bal. & Match. & Rev. & Bal. & Match. \\
\midrule
Standard (No Aug) & 0.113 & 0.549 & \textbf{0.948} & 0.704 & 0.390 & 0.166 & 0.762 & \underline{0.789} & \textbf{0.814} & 0.235 & 0.219 & 0.208 \\
BC & 0.113 & 0.549 & \textbf{0.948} & 0.779 & 0.393 & \textbf{0.087} & 0.762 & \underline{0.789} & \textbf{0.814} & \underline{0.225} & \textbf{0.204} & \textbf{0.186} \\
BC (shortcut-cond.) & 0.489 & 0.576 & 0.543 & 0.323 & 0.214 & 0.237 & 0.564 & \underline{0.791} & 0.568 & 0.315 & \underline{0.206} & 0.310 \\
CC & 0.113 & 0.549 & \textbf{0.948} & 0.712 & 0.390 & 0.156 & 0.762 & \underline{0.789} & \textbf{0.814} & 0.248 & 0.230 & 0.215 \\
CC (shortcut-cond.) & 0.187 & 0.556 & 0.904 & 0.596 & 0.360 & 0.217 & \textbf{0.806} & \underline{0.789} & 0.765 & 0.293 & 0.285 & 0.296 \\
OOC & 0.552 & 0.582 & -- & 0.242 & 0.222 & -- & 0.777 & 0.781 & -- & \textbf{0.214} & \underline{0.210} & -- \\
\cmidrule(lr){1-13}
Uniform CTTA & 0.609 & 0.609 & 0.609 & \textbf{0.152} & \textbf{0.152} & 0.152 & 0.792 & \textbf{0.792} & 0.792 & \underline{0.215} & \underline{0.215} & 0.215 \\
Neutral CTTA (removal) & \underline{0.619} & \underline{0.619} & 0.619 & 0.363 & 0.363 & 0.363 & 0.785 & 0.785 & 0.786 & \underline{0.218} & \underline{0.218} & 0.217 \\
Neutral CTTA (unsure) & 0.606 & 0.606 & 0.606 & 0.307 & 0.307 & 0.307 & -- & -- & -- & -- & -- & -- \\
Model-Weighted CTTA & \textbf{0.621} & \textbf{0.621} & 0.621 & 0.364 & 0.364 & 0.364 & 0.793 & \textbf{0.792} & 0.792 & \underline{0.222} & 0.223 & 0.224 \\
Select CTTA  & 0.609 & 0.609 & 0.609 & \textbf{0.152} & \textbf{0.152} & 0.152 & 0.792 & \textbf{0.792} & 0.792 & \underline{0.215} & \underline{0.215} & 0.215 \\
\bottomrule
\end{tabular}
}
\end{table}

\subsection{Shortcut Strength}\label{app:shortcut-strength}

Table~\ref{tab:app-strength} reports results for models fine-tuned at shortcut strengths $p \in \{0.50, 0.75, 0.95\}$, each evaluated on the reversed split of its own setting. At $p = 0.50$ the reversed and matched splits coincide. The balanced fine-tuning reference is the same balanced fine-tune at every $p$.

\begin{table}[ht]
\centering
\caption{Shortcut strength sweep on the reversed split. The mean over the six models is shown. Bold: best in the column. Underlined: not significantly different from the best (95\% paired bootstrap interval of the difference includes zero). At $p = 0.50$ the shortcut is absent from training. As shortcut strength increases, standard predictors and calibration baselines degrade sharply in most cases, while CTTA largely preserves AUROC and calibration. Select-CTTA consistently provides strong performance across tasks, shortcut strengths, and fine-tuning objectives, often approaching balanced training. }
\label{tab:app-strength}
\small\setlength{\tabcolsep}{3pt}
\resizebox{\textwidth}{!}{\begin{tabular}{lcccccccccccc}
\toprule
 & \multicolumn{6}{c}{Math (Sycophancy)} & \multicolumn{6}{c}{Occupation (Gender)} \\
\cmidrule(lr){2-7}\cmidrule(lr){8-13}
 & \multicolumn{3}{c}{AUROC $\uparrow$} & \multicolumn{3}{c}{ECE $\downarrow$} & \multicolumn{3}{c}{AUROC $\uparrow$} & \multicolumn{3}{c}{ECE $\downarrow$} \\
\cmidrule(lr){2-4}\cmidrule(lr){5-7}\cmidrule(lr){8-10}\cmidrule(lr){11-13}
Method & $p = 0.50$ & $p = 0.75$ & $p = 0.95$ & $p = 0.50$ & $p = 0.75$ & $p = 0.95$ & $p = 0.50$ & $p = 0.75$ & $p = 0.95$ & $p = 0.50$ & $p = 0.75$ & $p = 0.95$ \\
\midrule
\multicolumn{13}{@{}l}{\textit{All-tokens loss}} \\
Standard (No Aug) & 0.706 & 0.380 & 0.062 & 0.132 & 0.438 & 0.894 & \underline{0.949} & 0.878 & 0.271 & 0.069 & 0.102 & 0.572 \\
BC & 0.706 & 0.380 & 0.062 & \underline{0.067} & 0.441 & 0.895 & \underline{0.949} & 0.878 & 0.271 & \underline{0.057} & \textbf{0.062} & 0.581 \\
BC (shortcut-cond.) & 0.711 & 0.636 & 0.436 & \textbf{0.066} & \textbf{0.064} & 0.201 & \underline{0.949} & 0.857 & 0.550 & \textbf{0.055} & 0.074 & 0.262 \\
CC & 0.706 & 0.380 & 0.062 & 0.162 & 0.437 & 0.860 & \underline{0.949} & 0.878 & 0.271 & 0.077 & 0.076 & 0.577 \\
CC (shortcut-cond.) & 0.687 & 0.532 & 0.111 & 0.164 & 0.231 & 0.543 & 0.948 & \textbf{0.959} & \textbf{0.933} & 0.126 & 0.151 & 0.160 \\
\cmidrule(lr){1-13}
Uniform CTTA & \textbf{0.724} & 0.721 & 0.700 & 0.133 & 0.133 & 0.153 & \underline{0.950} & \underline{0.955} & 0.913 & 0.073 & 0.101 & 0.176 \\
Neutral CTTA (removal) & 0.681 & 0.659 & 0.659 & 0.090 & 0.240 & 0.312 & 0.946 & 0.946 & 0.919 & 0.079 & 0.091 & 0.079 \\
Neutral CTTA (unsure) & 0.708 & 0.707 & 0.700 & 0.135 & 0.280 & 0.447 & -- & -- & -- & -- & -- & -- \\
Model-Weighted CTTA & \textbf{0.724} & \textbf{0.728} & \textbf{0.745} & 0.135 & 0.080 & \textbf{0.141} & \textbf{0.951} & \underline{0.958} & \underline{0.930} & 0.068 & \underline{0.079} & \textbf{0.051} \\
Select CTTA & \textbf{0.724} & \underline{0.727} & 0.725 & 0.133 & 0.120 & 0.152 & \underline{0.950} & \underline{0.957} & \underline{0.930} & 0.073 & 0.086 & \textbf{0.051} \\
\midrule
\textit{Balanced Training (Reference)} & \textit{0.703} & \textit{0.694} & \textit{0.688} & \textit{0.139} & \textit{0.134} & \textit{0.132} & \textit{0.949} & \textit{0.942} & \textit{0.938} & \textit{0.069} & \textit{0.065} & \textit{0.067} \\
\midrule
\multicolumn{13}{@{}l}{\textit{Completion-only loss}} \\
Standard (No Aug) & \textbf{0.783} & 0.585 & 0.077 & 0.048 & 0.237 & 0.893 & \underline{0.950} & 0.837 & 0.378 & 0.076 & 0.082 & 0.506 \\
BC & \textbf{0.783} & 0.585 & 0.077 & \textbf{0.033} & 0.230 & 0.893 & \underline{0.950} & 0.837 & 0.378 & \underline{0.062} & 0.080 & 0.506 \\
BC (shortcut-cond.) & \textbf{0.783} & 0.710 & 0.430 & \underline{0.034} & \textbf{0.092} & 0.201 & \underline{0.950} & 0.798 & 0.613 & \textbf{0.061} & 0.095 & 0.325 \\
CC & \textbf{0.783} & 0.585 & 0.077 & 0.102 & 0.235 & 0.820 & \underline{0.950} & 0.837 & 0.378 & 0.076 & 0.121 & 0.485 \\
CC (shortcut-cond.) & 0.780 & 0.780 & 0.430 & 0.159 & 0.139 & 0.213 & 0.950 & \textbf{0.893} & 0.890 & 0.103 & 0.113 & 0.204 \\
\cmidrule(lr){1-13}
Uniform CTTA & \textbf{0.783} & \textbf{0.796} & \textbf{0.676} & 0.047 & 0.112 & \textbf{0.155} & \textbf{0.951} & 0.886 & \textbf{0.945} & 0.072 & \underline{0.061} & 0.151 \\
Neutral CTTA (removal) & 0.781 & 0.772 & 0.662 & 0.046 & 0.184 & 0.369 & \underline{0.950} & 0.885 & \underline{0.942} & 0.078 & \textbf{0.060} & \textbf{0.117} \\
Neutral CTTA (unsure) & \underline{0.782} & 0.788 & \underline{0.673} & 0.047 & 0.226 & 0.451 & -- & -- & -- & -- & -- & -- \\
Model-Weighted CTTA & \underline{0.782} & 0.783 & 0.654 & 0.048 & 0.124 & 0.315 & \textbf{0.951} & 0.877 & 0.814 & 0.073 & 0.069 & 0.229 \\
Select CTTA & \textbf{0.783} & \textbf{0.796} & \textbf{0.676} & 0.047 & 0.112 & \textbf{0.155} & \textbf{0.951} & 0.886 & \textbf{0.945} & 0.072 & \underline{0.061} & 0.151 \\
\midrule
\textit{Balanced Training (Reference)} & \textit{0.788} & \textit{0.790} & \textit{0.791} & \textit{0.064} & \textit{0.066} & \textit{0.066} & \textit{0.950} & \textit{0.947} & \textit{0.945} & \textit{0.076} & \textit{0.076} & \textit{0.078} \\
\bottomrule
\end{tabular}
}
\end{table}

\subsection{Accuracy and Cross-Entropy}\label{app:metrics}

Tables~\ref{tab:app-metrics-ft} and~\ref{tab:app-metrics-zs} report accuracy at threshold $0.5$ and cross-entropy on the reversed split for the $p = 0.95$ setting for fine-tuned and zero-shot models.

\begin{table}[ht]
\centering
\caption{Accuracy and cross-entropy (CE) of the fine-tuned models on the reversed split ($p = 0.95$), mean over the six models. Bold: best in the column. Underlined: not significantly different from the best (95\% paired bootstrap interval of the difference includes zero). In most cases, CTTA substantially improves accuracy and cross-entropy over the unaugmented model and standard calibration baselines. Model-Weighted CTTA is strongest under all-tokens loss but degrades substantially under completion-only loss. Select-CTTA obtains some of the benefits of Model-Weighted CTTA by selecting it in some cases in the all-tokens setting, and it avoids its failures by consistently choosing Uniform CTTA in the completion-only setting.}
\label{tab:app-metrics-ft}
\small\setlength{\tabcolsep}{4pt}
\resizebox{\textwidth}{!}{\begin{tabular}{lcccccccc}
\toprule
 & \multicolumn{4}{c}{All-tokens loss} & \multicolumn{4}{c}{Completion-only loss} \\
\cmidrule(lr){2-5}\cmidrule(lr){6-9}
 & \multicolumn{2}{c}{Math (Sycophancy)} & \multicolumn{2}{c}{Occupation (Gender)} & \multicolumn{2}{c}{Math (Sycophancy)} & \multicolumn{2}{c}{Occupation (Gender)} \\
\cmidrule(lr){2-3}\cmidrule(lr){4-5}\cmidrule(lr){6-7}\cmidrule(lr){8-9}
Method & Acc $\uparrow$ & CE $\downarrow$ & Acc $\uparrow$ & CE $\downarrow$ & Acc $\uparrow$ & CE $\downarrow$ & Acc $\uparrow$ & CE $\downarrow$ \\
\midrule
Standard (No Aug) & 0.059 & 2.937 & 0.292 & 1.830 & 0.064 & 3.014 & 0.386 & 2.180 \\
BC & 0.058 & 2.934 & 0.283 & 1.792 & 0.066 & 3.012 & 0.382 & 2.176 \\
BC (shortcut-cond.) & 0.496 & 0.715 & 0.513 & 0.917 & 0.462 & 0.734 & 0.557 & 1.564 \\
CC & 0.078 & 2.990 & 0.286 & 1.859 & 0.076 & 3.125 & 0.431 & 2.412 \\
CC (shortcut-cond.) & 0.224 & 1.198 & 0.776 & 0.450 & \textbf{0.549} & 0.770 & 0.740 & 0.770 \\
\cmidrule(lr){1-9}
Uniform CTTA & 0.519 & 0.718 & 0.783 & 0.503 & 0.521 & \textbf{0.720} & \textbf{0.868} & \textbf{0.435} \\
Neutral CTTA (removal) & 0.525 & 1.060 & 0.822 & 0.390 & 0.530 & 1.248 & 0.849 & \underline{0.480} \\
Neutral CTTA (unsure) & 0.513 & 1.570 & -- & -- & 0.513 & 1.633 & -- & -- \\
Model-Weighted CTTA & \textbf{0.619} & \textbf{0.652} & \textbf{0.848} & \textbf{0.337} & 0.522 & 1.191 & 0.680 & 0.999 \\
Select CTTA & 0.597 & 0.667 & \textbf{0.848} & \textbf{0.337} & 0.521 & \textbf{0.720} & \textbf{0.868} & \textbf{0.435} \\
\midrule
\textit{Balanced Training (Reference)} & \textit{0.571} & \textit{0.670} & \textit{0.855} & \textit{0.343} & \textit{0.715} & \textit{0.552} & \textit{0.863} & \textit{0.365} \\
\bottomrule
\end{tabular}
}
\end{table}

\begin{table}[ht]
\centering
\caption{Accuracy and cross-entropy (CE) of the zero-shot models on the reversed split. Bold: best in the column. Underlined: not significantly different from the best (95\% paired bootstrap interval of the difference includes zero). Select-CTTA consistently chooses Uniform CTTA and obtains top performance across all metrics and tasks.}
\label{tab:app-metrics-zs}
\small\setlength{\tabcolsep}{4pt}
\begin{tabular}{lcccc}
\toprule
 & \multicolumn{2}{c}{Math (Sycophancy)} & \multicolumn{2}{c}{Occupation (Gender)} \\
\cmidrule(lr){2-3}\cmidrule(lr){4-5}
Method & Acc $\uparrow$ & CE $\downarrow$ & Acc $\uparrow$ & CE $\downarrow$ \\
\midrule
Standard (No Aug) & 0.278 & 5.353 & 0.669 & 2.073 \\
BC & 0.159 & 5.037 & \textbf{0.701} & \underline{1.780} \\
BC (shortcut-cond.) & 0.468 & 1.689 & 0.615 & 2.863 \\
CC & 0.254 & 5.275 & 0.672 & 2.113 \\
CC (shortcut-cond.) & 0.334 & 3.761 & 0.653 & 2.195 \\
OOC & 0.532 & 1.463 & 0.684 & \textbf{1.711} \\
\cmidrule(lr){1-5}
Uniform CTTA & \textbf{0.565} & \textbf{0.825} & \underline{0.692} & \underline{1.734} \\
Neutral CTTA (removal) & 0.558 & 2.939 & \underline{0.687} & 1.809 \\
Neutral CTTA (unsure) & \underline{0.564} & 2.118 & -- & -- \\
Model-Weighted CTTA & 0.558 & 1.853 & \underline{0.695} & 1.840 \\
Select CTTA & \textbf{0.565} & \textbf{0.825} & \underline{0.692} & \underline{1.734} \\
\bottomrule
\end{tabular}

\end{table}

\subsection{Discussion of Balanced Training Comparison}\label{app:bal-train-compare}

In Table~\ref{tab:ft-results}, Model-Weighted CTTA sometimes outperforms the shortcut-balanced training reference in the all-tokens setting. We hypothesize that this is because Model-Weighted CTTA combines two distinct sources of label information, whereas the balanced reference relies on only one. Both methods use the classification predictions of a trained LLM. These predictions are more reliable for the balanced reference, which is trained on shortcut-balanced data ($p=0.50$), than for Model-Weighted CTTA, which uses a model trained with a strong shortcut ($p=0.95$). However, Model-Weighted CTTA also incorporates label information through its likelihood-based weights. Because the shortcut is strongly associated with the label during training, the model assigns higher likelihood to counterfactual edits whose shortcut values agree with the label suggested by the input's content. The resulting weights therefore provide an indirect prediction of the label. Although each of the two signals used by Model-Weighted CTTA may be less accurate than the balanced reference's classification predictions, they may make different errors. Combining these complementary signals can therefore allow Model-Weighted CTTA to outperform the balanced reference.

\end{document}